\documentclass[lettersize,journal]{IEEEtran}
\usepackage{amsmath,amsfonts}
\usepackage{array}
\usepackage{textcomp}
\usepackage{stfloats}
\usepackage{url}
\usepackage{verbatim}
\usepackage{graphicx}
\usepackage{cite}
\usepackage{bm}
\usepackage{multirow}
\usepackage{lscape} 
\usepackage{tabularx} 
\usepackage{comment}
\usepackage{color}
\usepackage{algorithm}
\usepackage[noend]{algpseudocode} 
\usepackage{amssymb,amsmath}
\usepackage{amsthm} 
\usepackage{xspace}
\usepackage{hyperref}

\usepackage[font=normal]{caption}
\usepackage{subcaption}

\newtheorem{theorem}{Theorem}
\newtheorem{lemma}{Lemma}

\newtheorem{remark}{Remark}

\newcommand{\tri}{$\pi_i[\underline{t}_i,\overline{t}_i]$\xspace}
\newcommand{\trj}{$\pi_j[\underline{t}_j,\overline{t}_j]$\xspace}
\newcommand{\timei}{$[\underline{t}_i,\overline{t}_i]$\xspace}
\newcommand{\timej}{$[\underline{t}_j,\overline{t}_j]$\xspace}

\begin{document}

\title{C$^*$: A New Bounding Approach for the Moving-Target Traveling Salesman Problem}


\author{Allen George Philip$^{1}$, Zhongqiang Ren$^{2}$, Sivakumar Rathinam$^{1}$, and Howie Choset$^{3}$
	 \thanks{1. Allen George Philip and Sivakumar Rathinam are with Texas A\&M University, College Station, TX 77843, USA.}
	 \thanks{2. Zhongqiang Ren is with Shanghai Jiao Tong University, 800 Dongchuan Road, Shanghai, China.}
  \thanks{3. Howie Choset is with Carnegie Mellon University, 5000 Forbes Ave., Pittsburgh, PA 15213, USA.}%
}





\newcommand{\blue}{\color{black}}
\newcommand{\red}{\color{red}}
\newcommand{\green}{\color{green}}

\newcommand\abbrCstarVars{{\it C$^*$}\xspace}
\newcommand\abbrCstarLinear{{\it C$^*$-Linear}\xspace}
\newcommand\abbrCstar{{\it C$^*$}\xspace}
\newcommand\abbrCstardash{{\it C$^*$-Lite}\xspace}
\newcommand\abbrCstarSampling{{\it C$^*$-Sampling}\xspace}
\newcommand\abbrCstarGeometric{{\it C$^*$-Geometric}\xspace}
\newcommand\abbrLinInst{simple\xspace}
\newcommand\abbrPWLinInst{complex\xspace}
\newcommand\abbrDubinsInst{generic\xspace}
\newcommand\abbrSamplingPar{$k$\xspace}
\newcommand\abbrSamplingGapTol{$\epsilon$\xspace}
\newcommand\abbrRT{RT\xspace}
\newcommand\abbrtotalRT{Total RT\xspace}
\newcommand\abbrGRT{Graph Gen RT\xspace}
\newcommand\abbrperDev{$\%$ deviation\xspace}
\newcommand\abbrsftAlg{{\it AlgoSFT}\xspace}

\newcommand\abbrEFAT{EFAT\xspace}
\newcommand\abbrEFA{EFA\xspace}
\newcommand\abbrLFDT{LFDT\xspace}
\newcommand\abbrLFD{LFD\xspace}
\newcommand\abbrTFRT{graph generation R.T.\xspace}
\newcommand\abbrEGTSPRT{E-GTSP-RT\xspace}
\newcommand\abbrtotRT{total R.T.\xspace}
\newcommand\abbrperdev{$\%$ deviation\xspace}
\newcommand\abbrSOCPRT{run-time for SOCP\xspace}

\newcommand\algname[1]{\textsf{#1}\xspace}
\newcommand\ERCAstar{\algname{ERCA*}}
\newcommand\ERCAstarNaive{\algname{ERCA*-Naive}}
\newcommand\EMOAstar{\algname{EMOA*}}
\newcommand\RCBDAstar{\algname{RCBDA*}}
\newcommand\PulseAlg{\algname{Pulse}}
\newcommand\BiPulseAlg{\algname{BiPulse}}

\newcommand\procedurename[1]{\textsl{#1}}
\newcommand\IsPrunedByFront{\procedurename{IsPrunByFront}}
\newcommand\IsPrunedByResour{\procedurename{IsPrunByResour}}
\newcommand\FilterAndAddFront{\procedurename{FilterAndAddFront}}
\newcommand\ProcFilter{\procedurename{Filter}}

\newcommand\FrontierSet{\mathcal{F}}
\newcommand\DomSymbol{\preceq}
\newcommand\NonDomSymbol{\npreceq}
\newcommand\LexLessSymbol{<_{lex}}
\newcommand\LexLargerSymbol{>_{lex}}
\newcommand\ResourceLimit{\Vec{r}_{limit}}
\newcommand\NDSymbol{\mathcal{N}\mathcal{D}}


\maketitle

\begin{abstract}
We introduce a new bounding approach called Continuity* (\abbrCstar), which provides optimality guarantees for the Moving-Target Traveling Salesman Problem (MT-TSP). Our approach relaxes the continuity constraints on the agent's tour by partitioning the targets' trajectories into smaller segments. This allows the agent to arrive at any point within a segment and depart from any point in the same segment when visiting each target. This formulation enables us to pose the bounding problem as a Generalized Traveling Salesman Problem (GTSP) on a graph, where the cost of traveling along an edge requires solving a new problem called the Shortest Feasible Travel (SFT). {\blue We present various methods for computing bounds for the SFT problem, leading to several variants of \abbrCstar.} We first prove that the proposed algorithms provide valid lower-bounds for the MT-TSP. Additionally, we provide computational results to validate the performance of all \abbrCstar variants on instances with up to 15 targets. For the special case where targets move along straight lines, we compare our \abbrCstar variants with a mixed-integer {\blue Second Order Conic Program (SOCP)} based method, the current state-of-the-art solver for the MT-TSP. While the SOCP-based method performs well on instances with 5 and 10 targets, \abbrCstar outperforms it on instances with 15 targets. For the general case, on average, our approaches find feasible solutions within approximately 4.5$\%$ of the lower-bounds for the tested instances.

\end{abstract}


\graphicspath{{./figure/}}

\section{Introduction}\label{intro}

The Traveling Salesman Problem (TSP) is one of the most important problems in optimization with several applications including  
unmanned vehicle planning \cite{oberlin2010today,liu2018efficient,ryan1998reactive, yu2002implementation}, transportation and delivery \cite{ham2018integrated}, monitoring and surveillance \cite{venkatachalam2018two,saleh2004design}, disaster management \cite{cheikhrouhou2020cloud}, precision agriculture \cite{conesa2016mix}, and search and rescue \cite{zhao2015heuristic, brumitt1996dynamic}. Given a set of target locations (or targets) and the cost of traveling between any pair of targets, the TSP aims to find a shortest tour for a vehicle to visit each of the targets exactly once. In this paper, we consider a natural generalization of the TSP where the targets are mobile and traverse along known trajectories. The targets also have time-windows during which they must be visited. This generalization is motivated by applications including monitoring and surveillance \cite{deMoraes2019,wang2023moving,marlow2007travelling,maskooki2023bi}, missile defense \cite{helvig2003,stieber2022,smith2021assessment}, fishing \cite{groba2015solving,granado2024fishing}, human evacuation \cite{sriniketh2023robot}, and dynamic target tracking \cite{englot2013efficient}, where vehicles are required to visit or monitor a set of mobile targets. We refer to this generalization as the Moving-Target Traveling Salesman Problem, or MT-TSP for short (refer to Fig.~\ref{fig:feasible}). The focus of this paper is on the optimality guarantees for the MT-TSP.


\vspace{1mm}
Targets are typically assumed to move at a speed equal to or slower than the agent's speed since the agent is anticipated to visit or intercept each target \cite{helvig2003}. When the speed of every target decreases to 0, the Moving-Target Traveling Salesman Problem (MT-TSP) simplifies to the standard Traveling Salesman Problem (TSP). Therefore, MT-TSP is a generalization of the TSP and is NP-Hard. {\blue In addition, the MT-TSP considered in this paper includes time-window constraints for visiting the targets. These additional constraints make the problem even more challenging, as finding a feasible solution to the TSP with time-window constraints is NP-hard, even for stationary targets \cite{savelsbergh1985local}.} Unlike the TSP which has been extensively studied, the current literature on MT-TSP is limited. 

\begin{figure}[t]
    \centering    
    \includegraphics[width=\linewidth]{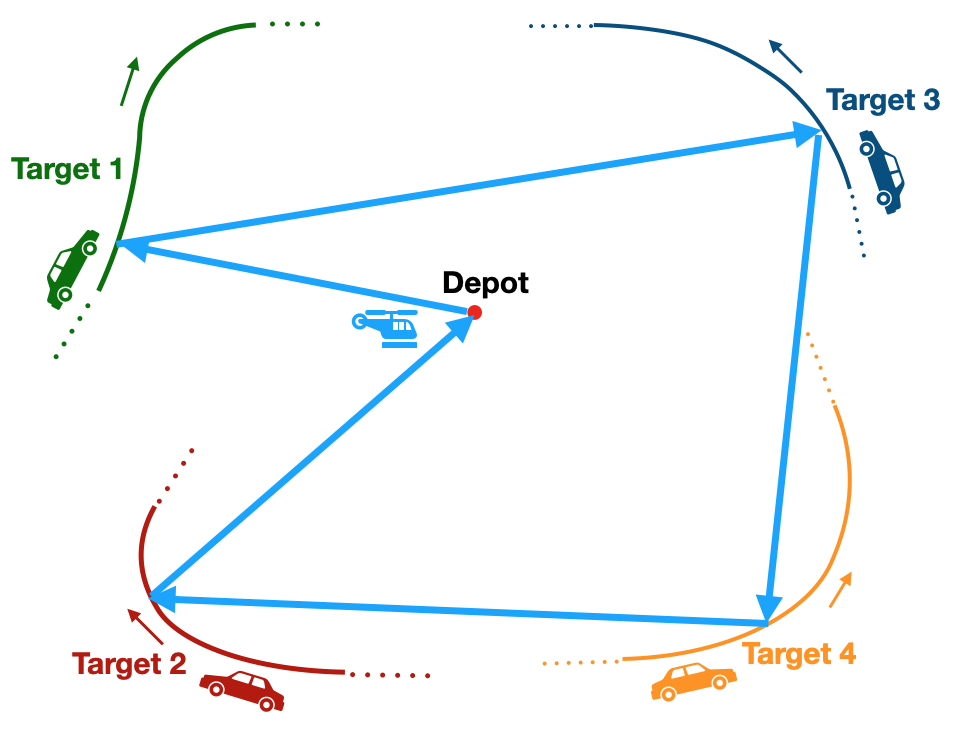}
    \caption{A feasible solution for an instance of the MT-TSP with four targets. The blue lines shows the path of the vehicle. Also, the colored solid segments for each target indicates the part of its trajectory corresponding to its time-windows when the vehicle can visit the target.}
    \label{fig:feasible}
\end{figure}

\subsection{Literature Review}

Exact and approximation algorithms are available in the literature for some special cases of the MT-TSP where the targets are assumed to move along straight lines with constant speeds. In \cite{chalasani1999approximating}, Chalasani and Motvani propose a $(5(1+v)/3(1-v))$-approximation algorithm for the case where all the targets move in the same direction with the same speed $v$. Hammar and Nilsson \cite{hammar1999}, also present a $(1+(\epsilon/(1-v))$-approximation algorithm for the same case. They also show that the MT-TSP cannot be approximated better than a factor of $2^{\Omega(\sqrt{n})}$ by a polynomial time algorithm (unless $P=NP$), where $n$ is the number of targets. In \cite{helvig2003}, Helvig et al. develop an exact $O(n)$-time algorithm for the case when the moving targets are restricted to a single line (SL-MT-TSP). Also, for the case where most of the targets are static and at most $O(log(n)/log(log(n)))$ targets are moving out of the $n$ targets, they propose a $(2+\epsilon)$-approximation algorithm. An exact algorithm is also presented \cite{helvig2003} for the MT-TSP with resupply where the agent must return to the depot after visiting each target; here, the targets are assumed to be far away from the depot or slow, and move along lines through the depot, towards or away from it. In \cite{stieber2022}, Stieber, and F{\"u}genschuh formulate the MT-TSP as a mixed-integer {\blue Second Order Conic Program} (SOCP) by relying on the key assumption that the targets travel along straight lines. Also, multiple agents are allowed, agents are not required to return to the depot, and each target has to be visited exactly once within its visibility window by one of the agents. Optimal solutions to the MT-TSP are then found for this special case. Recently, we have also developed a new formulation \cite{philip2024mixed} for a special case when the targets travel along straights lines and the objective is to minimize the distance traveled by the agent; this paper on the other hand deals with minimizing travel time for a generalization of the MT-TSP with multiple time-windows.

\vspace{1mm}
 Apart from the above methods that provide optimality guarantees to some special cases of the MT-TSP, feasible solutions can be obtained using heuristics as shown in \cite{bourjolly2006orbit, choubey2013, deMoraes2019, englot2013efficient, groba2015solving, jiang2005tracking, marlow2007travelling, ucar2019meta, wang2023moving, smith2021assessment, sriniketh2023robot}. However, these approaches do not show how far the feasible solutions are from the optimum.


\vspace{1mm}
A few variants of the MT-TSP and related problems have also been addressed in the literature. In \cite{hassoun2020}, Hassoun et al. suggest a dynamic programming algorithm to find an optimal solution to a variant of the SL-MT-TSP where the targets move at the same speeds and may appear at different times. Masooki and Kallio in \cite{maskooki2023bi} address a bi-criteria variant of the MT-TSP where the number of targets vary with time and their motion is approximated using (discontinuous) step functions of time. 

\vspace{1mm}

\subsection{Our work and contributions}

\blue In this article, we consider a generalization of the MT-TSP where each target moves along piecewise-linear segments or Dubins \cite{Dub} curves. \color{black} Each target may also be associated with time-windows during which the vehicle must visit the target. {\blue When targets move along generic trajectories, such as those considered in this paper, no algorithms currently exist for finding the optimum to the MT-TSP. When optimal solutions are difficult to obtain, one can develop algorithms that can find feasible solutions which provide upper bounds, and lower-bounding algorithms that provide tight underestimates to the optimum. Optimality guarantees can then be obtained by comparing the upper and lower-bounds.}


One way to generate feasible solutions to this problem is to sample a discrete set of times from the (planning) time horizon, and then consider the corresponding set of locations for each target; given a pair of targets and their sampled times, one can readily check for feasibility of travel and compute the travel costs between the targets. A solution can then be obtained for the MT-TSP by posing it as a {\it Generalized} TSP (GTSP) \cite{laporte1987generalized}  where the objective is to find an optimal TSP tour that visits exactly one (sampled) location for each target. While this approach can produce feasible solutions and upper-bounds, it may not find the optimum or lower-bounds for the MT-TSP. {\blue Therefore, we seek to develop methods that primarily focus on finding tight lower-bounds in this paper. }

In the special case where each target travels along a straight line, the SOCP-based formulation in \cite{stieber2022} can be used to find the optimum. For the general case where each target travels along a trajectory made of {\blue piecewise-linear segments or Dubins curves}, currently, we do not know of any method in the literature that can find the optimum or provide tight lower-bounds to the optimum for the MT-TSP. In this article, we develop a new approach called Continuity* (\abbrCstar) to answer this question.


\vspace{1mm}
\abbrCstar relies on the following key ideas. First, we relax the continuity of the trajectory of the agent and allow it to be discontinuous whenever it reaches the trajectory of a target. We do this by partitioning the trajectory of each target into smaller intervals \footnote{Since each target travels at a constant speed, distance traveled along the trajectory has one to one correspondence to the time elapsed.} and allow the agent to arrive at any point in an interval and depart from any point from the same interval. We then construct a graph $\mathcal{G}$ where all the nodes (intervals) corresponding to each target are grouped into a cluster, and any two nodes belonging to distinct clusters is connected by an edge. Next, the cost of traveling any edge is obtained by solving a Shortest Feasible Travel (SFT) problem between two intervals corresponding to distinct targets. {\blue Specifically, given two distinct targets $i$ and $j$, and parts of their trajectories, $\bar{\pi}_i$ and $\bar{\pi}_j$, corresponding to two intervals, the SFT problem aims to find a time $t_i$ to depart from $\bar{\pi}_i$ and feasibly reach $\bar{\pi}_j$ at time $t_j$ such that $t_j - t_i$ is minimized.}

Once all the travel costs are computed, we formulate a Generalized TSP (GTSP) \cite{laporte1987generalized} in $\mathcal{G}$ which aims to find a tour such that exactly one node from each cluster is visited by the tour and the sum of the costs of the edges in the tour is minimum. We then show that our approach provides lower-bounds for the MT-TSP. As the number of partitions or discretizations of each target's trajectory increases, the lower-bounds get better and converge to the optimum. 

\vspace{1mm}
{\blue In addition to solving the SFT problem to optimality for targets moving along piecewise-linear segments, we also provide three simple and fast methods for computing bounds to the cost of an edge in $\mathcal{G}$ for more generic target trajectories. In this way, we develop several variants of \abbrCstar.} We also show how feasible solutions can be constructed from the lower-bounds, though this may not be always possible\footnote{If tight time-window constraints are present, finding feasible solutions is difficult in any case, whether we use \abbrCstar or not.} if challenging time-window constraints are present.

{\blue We provide extensive computational results to corroborate the performance of the variants of \abbrCstar  for instances with up to 15 targets. For the special case where targets travel along lines, we compare our \abbrCstar variants with the SOCP-based method, which is the current state-of-the-art solver for MT-TSP. While \abbrCstar provides similar bounds compared to the SOCP-based method for all cases, \abbrCstar is an order of magnitude faster relative to the SOCP-based method for instances with 15 targets. For the general case, on average, our approaches find feasible solutions within $\approx~$4.5$\%$ of the lower-bounds for the tested instances.}

\section{Problem Definition}\label{problem}

Let $S := \{s_1, s_2, \cdots, s_n\}$ be the set of targets. {\blue All the targets and the agent move in a 2D Euclidean plane.} Each target $i \in S$ moves at a constant speed $v_{i}>0$, and follows a \emph{continuous trajectory} $\pi_{i}$. {\blue In this paper, $\pi_{i}$ is piecewise-linear (made of a finite set of line-segments) or a Dubins curve (made of circular arcs of a given turning radius and line-segments).} Consider an agent that moves at a speed no greater than $v_{max}$ at any time instant. There are no other dynamic constraints placed on the motion of the agent. The agent starts and ends its path at a location referred to as the depot.
We assume $v_{max}> v_{i}$ for all $i \in S$. Also, any target $i\in S$ is associated with a set of $k$ time-windows $[\underline{t}_{i, 1}, \overline{t}_{i, 1}], \cdots, [\underline{t}_{i, k}, \overline{t}_{i, k}]$ during which times the agent can visit the target. The objective of the MT-TSP is to find a tour for the agent such that
\begin{itemize}
    \item the agent starts and ends its tour at the depot $d$,
    \item the agent visits each target $i\in S$ exactly once within one of its specified time-windows $[\underline{t}_{i, 1}, \overline{t}_{i, 1}], \cdots, [\underline{t}_{i, k}, \overline{t}_{i, k}]$, and
    \item the travel time of the agent is minimized.
\end{itemize}

\blue
\section{Notations and Definitions}\label{sec:notation}

A \emph{\bf trajectory-point} for a target $i\in S$ is denoted by $\pi_{i}(t)$ and represents the position occupied by target $i$ at time $t$. 
A \emph{\bf trajectory-interval} for a target $i$ is denoted by \tri, and refers to the set of all the positions occupied by $i$ over the time interval \timei, where $\underline{t}_i \leq \overline{t}_i$. In the special case when $\underline{t}_i =\overline{t}_i$, \tri reduces to a trajectory-point and can be written as just $\pi_i(\underline{t}_i) $. Suppose the time interval \timei lies within another time interval $[\underline{t}_{i,1}, \overline{t}_{i,1}]$ ($i.e.$ \timei $\subseteq [\underline{t}_{i,1}, \overline{t}_{i,1}]$). Then, \tri is said to be a \emph{\bf trajectory-sub-interval} that lies within the trajectory-interval $\pi_{i}[\underline{t}_{i,1}, \overline{t}_{i,1}]$.

A \emph{\bf travel} from $\pi_{i}(t)$ to $\pi_{j}(t')$ denotes the event where the agent departs from $\pi_{i}(t)$ at time $t$, and arrives at $\pi_{j}(t')$ at time $t'$. This is the same as saying the agent departs from target $i$ at time $t$ and arrives at target $j$ at time $t'$. For travel from the depot $d$ to $\pi_{i}(t)$, the agent departs from $d$ at time $t_d = 0$, and for travel from $\pi_{i}(t)$ to $d$, the agent arrives at $d$ at some time $t_d \geq t$.

A {\bf travel is feasible} if the agent can complete it without exceeding its maximum speed $v_{max}$. Clearly, feasible travel requires the arrival time to be greater than or equal to the departure time. We define a feasible travel exists from some trajectory-interval \tri to another trajectory-interval \trj, if there exists some $t \in$\timei and some $t' \in$\timej such that the travel from $\pi_{i}(t)$ to $\pi_{j}(t')$ is feasible.  We also define a feasible travel exists from a trajectory-point $\pi_{i}(t_i)$ to a target trajectory $\pi_{j}$, if there exists some $t\geq 0$ such that travel from $\pi_{i}(t_i)$ to $\pi_{j}(t)$ is feasible. Feasible travels from a trajectory to a trajectory-point or feasible travels to and from a depot can be defined similarly. 

\vspace{1mm}
Now, we define the following optimization problems between any two targets $i$ and $j$ which we need to solve as part of our approach to the MT-TSP:
\begin{itemize}

\item {\bf Shortest Feasible Travel (SFT) problem from the trajectory-interval \tri to the trajectory-interval \trj:}
\begin{itemize}
 
 \item $\min_{t,t'} \ {(t'-t)}$ where $t \in$ \timei, $\ t' \in$ \timej, and travel from $\pi_{i}(t)$ to $\pi_{j}(t')$ is feasible. \\
\end{itemize}

\item {\bf Earliest Feasible Arrival Time (\abbrEFAT) problem from the trajectory-point $\pi_{i}(t_i)$ to the trajectory $\pi_{j}$:}

\begin{itemize}
    \item $\min t$ such that $t \geq 0$ and the travel from $\pi_{i}(t_i)$ to $\pi_{j}(t)$ is feasible.
    \end{itemize}
The {optimal time} (also referred to as EFAT) to the above problem is denoted as ${\bm{\mathcal{E}(t_i)}}$. In other words, after departing from $\pi_{i}(t_i)$, the agent can reach target $j$ at the earliest at the trajectory-point $\pi_{j}(\mathcal{E}(t_i))$. \\ 

\item {\bf Latest Feasible Departure Time (\abbrLFDT) problem from a trajectory $\pi_{i}$ to a trajectory-point $\pi_{j}(t_j)$:} 
\begin{itemize}
\item $\max t$ such that $t \geq 0$ and the travel from $\pi_{i}(t)$ to $\pi_{j}(t_j)$ is feasible. 
\end{itemize}
The optimal time (also referred to as LFDT) to the above problem is denoted as $\bm{\mathcal{L}(t_j)}$. In other words, if the agent needs to arrive at $\pi_j(t_j)$, the latest position it must depart from on $\pi_i$ is $\pi_i(\mathcal{L}(t_j))$.

 \end{itemize}

\section{\abbrCstar Algorithm}\label{method}

\begin{figure}[htb!]
\centering
  \begin{subfigure}[t]{.82\linewidth}
    \centering\includegraphics[width=1\linewidth]{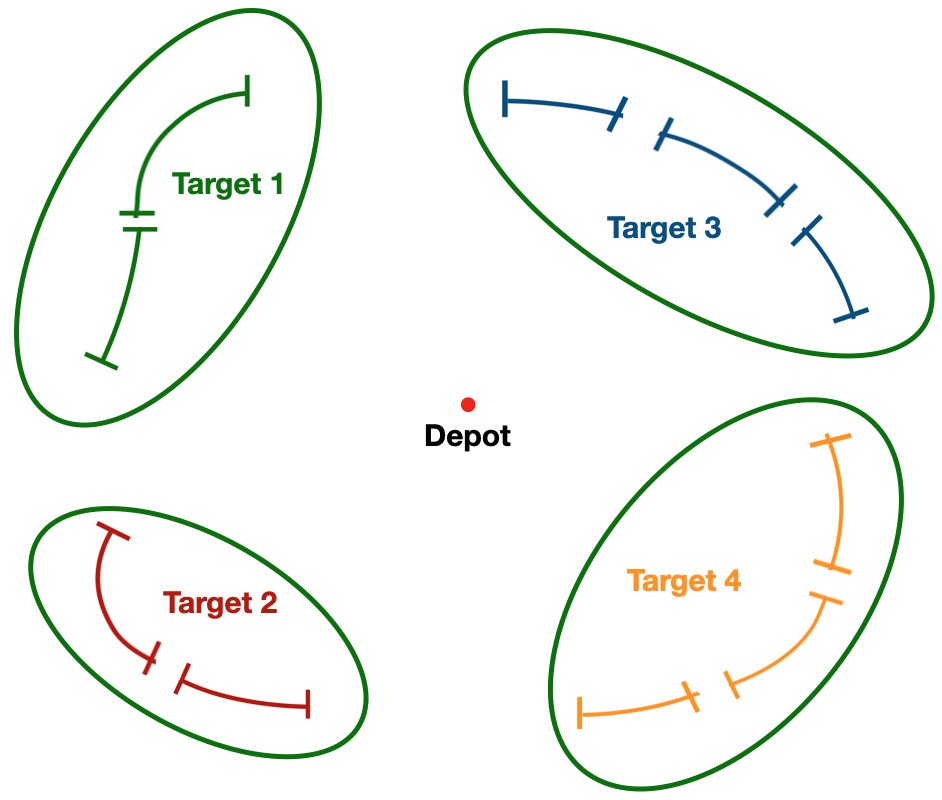}
    \caption{Partition the time-windows to create a cluster of trajectory-intervals for each target.}
    \label{fig:cstarA}
  \end{subfigure} \\ \vspace{.1cm}
  \begin{subfigure}[t]{.82\linewidth}
    \centering\includegraphics[width=1\linewidth]{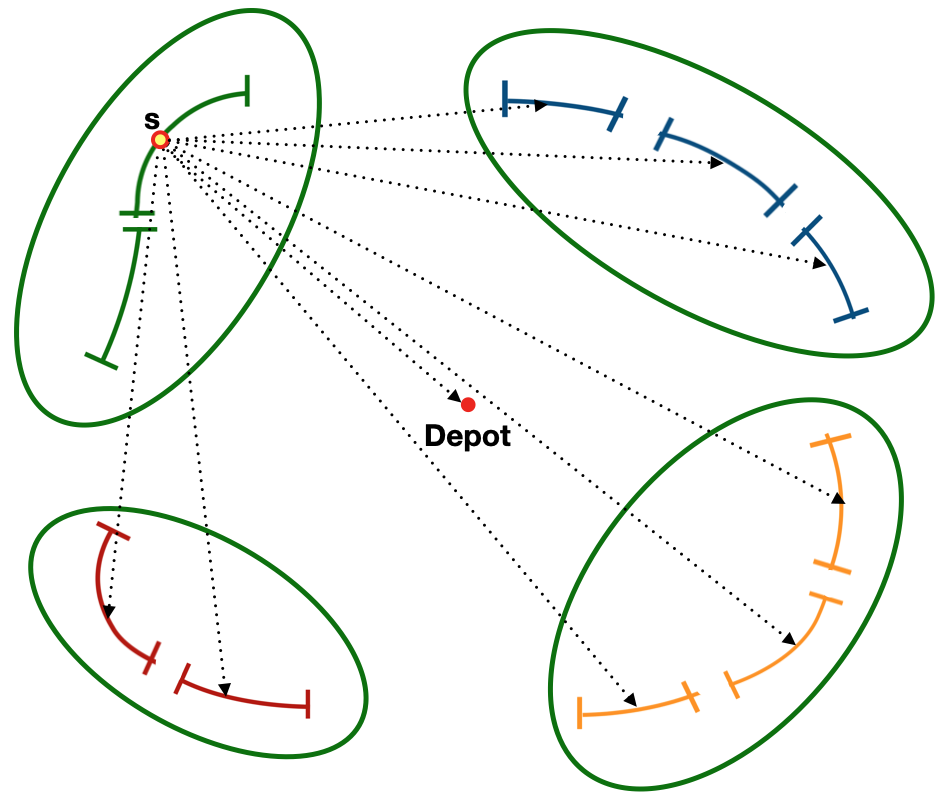}
    \caption{Construct a graph built on the clusters of trajectory-intervals corresponding to the targets. In this figure, we only show the directed edges from trajectory-interval $s$ to all the nodes outside the cluster containing $s$.}
     \label{fig:cstarB}
  \end{subfigure} \\ \vspace{.1cm}
  \begin{subfigure}[t]{.82\linewidth}
    \centering\includegraphics[width=1\linewidth]{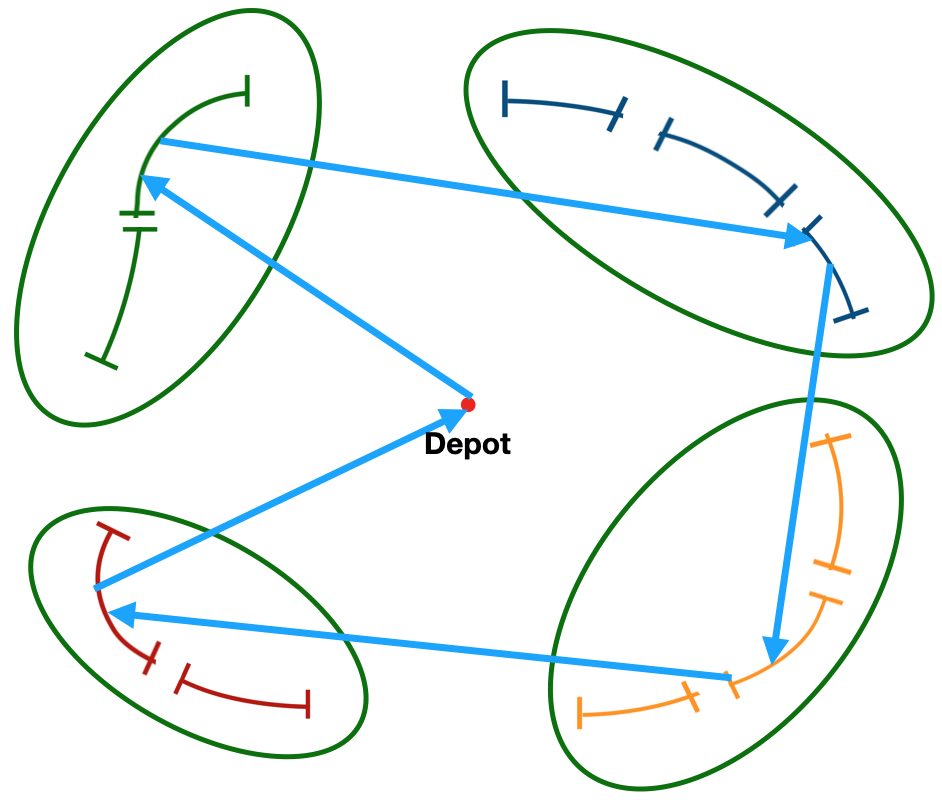}
    \caption{Solve a GTSP on the graph to find an optimal solution to the GTSP. Note the discontinuities in the optimal solution when the agent's path reaches a trajectory-interval (a node) in each cluster.}
     \label{fig:cstarC}
  \end{subfigure}
  \caption{Illustration of the key steps in the \abbrCstar Algorithm.}
  \label{fig:cstar}
\end{figure}

The following are the three key steps in the \abbrCstar Algorithm. These steps are also illustrated in Fig. \ref{fig:cstar}.

\begin{enumerate}
\item {\bf Partition the time-windows corresponding to each target:} The time-windows corresponding to each target in $S$ is first partitioned into smaller intervals\footnote{Here, without loss of generality, we assume that each time-window corresponding to a target is an integer multiple of $\Delta$.} of a given size (say $\Delta$). This partitioning step results in a cluster of trajectory-intervals corresponding to each target as shown in Fig. \ref{fig:cstarA}. Henceforth, each trajectory-interval in any of these clusters is also referred to as a node. Also, the cluster of nodes corresponding to the target $i\in S$ is denoted as $C_i$. In the next step, we will construct a graph using these nodes and relax the continuity of the agent's path when it reaches any one of the nodes corresponding to a target. \\

\item {\bf Construct a graph $\mathcal{G}$ with travel costs:} The graph $\mathcal{G}$ we construct is defined over the depot and the nodes in $C_i$, $i\in S$. Any two nodes $p,q$ in $\mathcal{G}$ are connected by directed edges if they belong to distinct clusters. Also, the depot is connected to all the remaining nodes in $\mathcal{G}$ through directed edges. Formally, $\mathcal{G}:=(V,E)$ where $V=\{d\}\cup \{p: p\in C_i, i\in S\}$ and $E:=\{(p,q): p \in C_i, q\in C_j, i,j\in S, i\neq j \} \cup \{(p,d),(d,p): p \in V\setminus \{d\}\}$. Each edge $(p,q)$ in $\mathcal{G}$ is associated with a non-negative, travel cost $l_{pq}$ which is obtained by using any one of the algorithms presented in section \ref{sec:edge_cost}. Since, the aim of this paper is to generate tight bounds, given an edge $(p,q)$, $l_{pq}$ must be a lower-bound on the cost of the SFT from node $p$ to node $q$.  \\

\item {\bf Solve the GTSP on $\mathcal{G}$ to find a bound}: The objective of the GTSP is to find a tour that starts and ends at the depot such that exactly one node is visited from each of the clusters corresponding to the targets and the sum of the travel costs is minimized. We will prove in Theorem~\ref{theorem:lbC*} that the optimum to the GTSP provides a lower-bound on the optimal cost to the MT-TSP. 
\end{enumerate}

\vspace{1mm}
\begin{theorem} \label{theorem:lbC*}
Suppose for any edge $(p,q)\in \mathcal{G}$, $l_{pq}$ denotes a lower-bound on the cost of the SFT from $p$ to $q$. Then, the optimal solution for the GTSP on $\mathcal{G}$ obtained in \abbrCstar provides a lower-bound on the optimum of the MT-TSP.
\end{theorem}

\begin{proof}
Consider an optimal solution to the MT-TSP. Let the sequence of nodes in $\mathcal{G}$ visited by this solution be $S^*:=({d,i_1, \cdots, i_n,d})$ and the corresponding arrival times be $({0,t_1, \cdots, t_n,t_{tour}})$. In this solution, the agent travels from $d$ to $\pi_{i_1}(t_1)$, then travels from $\pi_{i_k}(t_k)$ to $\pi_{i_{k+1}}(t_{k+1})$ for $k = 1, \cdots, n-1$ and finally travels from $\pi_{i_n}(t_n)$ to $d$ at speed $v_{max}$. The time taken by the agent to complete the tour is $t_{tour} = t_1 + \sum_{k=1}^{n-1} (t_{k+1} - t_{k}) + (t_{tour} - t_n)$. For any pair of adjacent nodes $(p,q)$ visited by the agent in this optimal solution, since $l_{pq}$ is a lower-bound on the SFT cost from $p$ to $q$, $l_{di_1} \leq t_1 $, $l_{i_ki_{k+1}} \leq t_{k+1}-t_k $ for $k = 1, \cdots, n-1$ and $l_{i_nd} \leq t_{tour} -t_{n}$. Therefore, for the sequence of nodes $S^*$ in the optimal solution, $l_{di_1} + \sum_{k=1}^{n-1}l_{i_ki_{k+1}} + l_{i_nd} \leq t_{tour}$. Since $S^*$ is also a feasible solution to the GTSP, the optimal cost obtained by solving the GTSP must be at most equal to $l_{di_1} + \sum_{k=1}^{n-1}l_{i_ki_{k+1}} + l_{i_nd} \leq t_{tour}$. Hence proved.
\end{proof}
\begin{remark} \label{remark:C*lb}
 As the number of partitions of the time-windows corresponding to the targets tend to infinity, the size of each time interval ($\Delta$) tends to 0 and that of the discontinuities in the agent's path tend to disappear. Since for any partition of the time-windows, \abbrCstar provides a lower-bound, the optimal cost of the relaxed MT-TSP converges asymptotically to the optimal cost of the MT-TSP as $\Delta\rightarrow 0$.
\end{remark}

\section{Algorithms for Computing Travel Costs}
\label{sec:edge_cost}

Given two trajectory-intervals (or nodes\footnote{If a node is generated from a depot, the algorithms in section \ref{sec:edge_cost} can still be applied by treating the depot as a stationary target.}) $p:=$\tri and $q:=$\trj, this section presents four bounding algorithms to estimate the travel cost $l_{pq}$ such that $l_{pq}$ is at most equal to the optimal SFT cost from \tri to \trj. Three of the algorithms can be applied to generic target trajectories while the fourth algorithm is specially optimized and tailored for piecewise-linear target trajectories. Before we present our algorithms, we introduce two conditions that, if satisfied, will address the trivial cases that do not require further optimization.

\begin{theorem} \label{thm:intintfeas}
 If the travel from $\pi_{i}(\underline{t}_i)$ to $\pi_{j}(\overline{t}_j)$ is not feasible, then the travel from \tri to \trj is not feasible, and vice-versa. 
\end{theorem}

\begin{proof}
We provide a proof by contraposition. If travel from $\pi_{i}(\underline{t}_i)$ to $\pi_{j}(\overline{t}_j)$ is feasible, then it readily follows that the travel from \tri to \trj is feasible. Now, let us show the other direction. If the travel from \tri to \trj is feasible, then there is a time $t^*_i \in$ \timei and a time $t^*_j \in$ \timej such that the travel from $\pi_i(t^*_i)$ to $\pi_j(t^*_j)$ is feasible. Let $v^*$ denote the speed of the agent during this travel. Since the agent can match the speed of any target, it can travel along the trajectory of target $i$ from $\pi_i(\underline{t}_i)$ to $\pi_i(t^*_i)$ at speed $v_i$, then travel from $\pi_i(t^*_i)$ to $\pi_j(t^*_j)$ at speed $v^*$, and finally travel along the trajectory of target $j$ from $\pi_j(t^*_j)$ to $\pi_j(\overline{t}_j)$ at speed $v_j$. As a result, the travel from  $\pi_{i}(\underline{t}_i)$ to $\pi_{j}(\overline{t}_j)$ is feasible. 
\end{proof}

\begin{theorem}\label{thm:optimal}
    If travel from $\pi_{i}(\overline{t}_i)$ to $\pi_{j}(\underline{t}_j)$ is feasible, then $\underline{t}_j-\overline{t}_i$ is the optimal cost of SFT from \tri to \trj.
\end{theorem}
\begin{proof}
    Note that the objective of SFT is to find a feasible solution such that $t_j-t_i$ is minimized subject to $t_j\in$\timej and $t_i\in$\timei. Therefore, a trivial lower-bound on the optimal cost to this SFT problem is  $\underline{t}_j-\overline{t}_i$. If travel from $\pi_{i}(\overline{t}_i)$ to $\pi_{j}(\underline{t}_j)$ is feasible, its travel cost will be equal to $\underline{t}_j-\overline{t}_i$ which matches the lower-bound. Hence, $\underline{t}_j-\overline{t}_i$ must be the optimal cost of the SFT from \tri to \trj.
\end{proof}

Based on the above theorems, we first check if the travel from $\pi_{i}(\underline{t}_i)$ to $\pi_{j}(\overline{t}_j)$ is feasible. If this condition is not satisfied, then from Theorem \ref{thm:intintfeas},  travel from \tri to \trj is also not feasible. In this case, each of the algorithms return a very large value for $l_{pq}$ and no further computations are necessary. Next, if travel from $\pi_{i}(\overline{t}_i)$ to $\pi_{j}(\underline{t}_j)$ is feasible, each algorithm sets $l_{pq}$ to be equal to  $\underline{t}_j-\overline{t}_i$ and no further optimization is required (from Theorem \ref{thm:optimal}). Therefore, we assume henceforth that travel from $\pi_{i}(\underline{t}_i)$ to $\pi_{j}(\overline{t}_j)$ is feasible, while travel from $\pi_{i}(\overline{t}_i)$ to $\pi_{j}(\underline{t}_j)$ is not feasible. Based on this assumption, we present our bounding algorithms. We obtain a variant of \abbrCstar based on the choice of the bounding algorithm used; consequently, each of the following subsections is named accordingly to match the corresponding \abbrCstar variant.

\subsection{\abbrCstardash} \label{simple lb}

\vspace{1mm}
We derive a simple bound based solely on the timing constraints. Regardless of the feasibility of travel from $\pi_{i}(\overline{t}_i)$ to $\pi_{j}(\underline{t}_j)$, as discussed in the proof of Theorem \ref{thm:optimal}, $\underline{t}_j - \overline{t}_i$ serves as a valid lower-bound for the SFT cost. Additionally, we know that the optimal SFT cost is always non-negative. Hence, in this variant of \abbrCstar, we set $l_{pq} := \max\{\underline{t}_j - \overline{t}_i, 0\}$.

\subsection{\abbrCstar-Geometric} \label{sec:C*geometric}

\vspace{1mm}
In this variant of \abbrCstar, we ignore the motion constraints of the target trajectories and only consider the shortest Euclidean distance between the points traveled in \tri and \trj. Specifically, let all the points traveled in \tri and \trj be denoted as $S_1$ and $S_2$ respectively. Let the shortest Euclidean distance between the sets $S_1$ and $S_2$ be defined as $dist(S_1,S_2):=\min_{x_1\in S_1, x_2\in S_2}\lVert x_1-x_2 \rVert _2$. Here, we set $l_{pq}:=\frac{dist(S_1,S_2)}{v_{max}}$ which is clearly a lower-bound on the SFT cost.

For the target trajectories considered in this paper, each segment of a trajectory is either a straight line or an arc with a given turning radius for the target. If \tri and \trj correspond to $k_i$ and $k_j$ segments, respectively, we compute $dist(S_1, S_2)$ by finding the shortest distance between any segment in $S_1$ and any segment in $S_2$, and then computing the minimum of all these optima. This requires $O(k_i k_j)$ computations. 

\subsection{\abbrCstar-Sampling}
This variant of \abbrCstar first partitions \timei into \abbrSamplingPar uniform sub-intervals, where \abbrSamplingPar is a sampling parameter. It then estimates the optimal SFT cost for departing from each sub-interval and selects the minimum. Formally, let the $p^{th}$ sub-interval of \timei be $[\underline{t}_{i,p}, \overline{t}_{i,p}]$. The optimal SFT cost is then given by $ \min_{p=1}^{k} \min_{t\in [\underline{t}_{i,p}, \overline{t}_{i,p}]} {( {\mathcal{E}}(t)-t)}$.


 Since ${\mathcal{E}}(t)$ monotonically increases with $t$ (later proved in Theorem \ref{thm:monotonic}), the optimal SFT cost is at least equal to $l_{pq}:=\min_{p=1}^{k} ({\mathcal{E}}(\underline{t}_{i,p})-\overline{t}_{i,p}) $. If travel is infeasible from $\pi_i(\underline{t}_{i,p})$ to $\pi_j(\bar{t}_j)$, then travel from $\pi_i(t)$ for any time $t\geq \underline{t}_{i,p}$ to $\pi_j$ is also infeasible (from Theorem \ref{thm:intintfeas}); in this case, ${\mathcal{E}}(\underline{t}_{i,p})$ is set to a very large value. Otherwise, if the trajectories consist of straight lines, ${\mathcal{E}}(\underline{t}_{i,p})$ is relatively easy to find using the formulae provided in the appendix. The number of steps required to compute $l_{pq}$ in this case is $O(kk_j)$ where $k_j$ is the number of line segments in $\pi_j$.

If the trajectories consist of more generic segments where ${\mathcal{E}}(\underline{t}_{i,p})$ is difficult to compute directly, we can bound ${\mathcal{E}}(\underline{t}_{i,p})$ using the following approach: We know that if the agent starts at $\pi_i(\underline{t}_{i,p})$, it cannot arrive at $\pi_j$ at any time earlier than ${\mathcal{E}}(\underline{t}_{i,p})$. Therefore, if there is a sub-interval $[\mathcal{E}_{lp},\mathcal{E}_{up}] \subseteq$ \timej such that the agent can start from $\pi_i(\underline{t}_{i,p})$ and reach $\pi_j(\mathcal{E}_{up})$ but cannot reach $\pi_j(\mathcal{E}_{lp})$, then ${\mathcal{E}}(\underline{t}_{i,p})\in [\mathcal{E}_{lp},\mathcal{E}_{up}]$. To refine this bound, we iteratively partition \timej into smaller intervals using binary search, ensuring that the above condition is satisfied and $|\mathcal{E}_{up}-\mathcal{E}_{lp}|\leq \epsilon$ where $\epsilon$ is another sampling parameter that can be adjusted. Therefore, in this general case, we define $l_{pq}:=\min_{p=1}^{k} ({\mathcal{E}}_{lp}-\overline{t}_{i,p})$. The sampling algorithm in this case will require $O(k\log_2{\frac{\Delta}{\epsilon}})$ steps where $\Delta$ is the size of \timej.

\subsection{\abbrCstar-Linear}\label{sec:C*linear}

This variant of  \abbrCstar finds the optimum for SFT, and can be applied when both trajectory intervals,  \tri and \trj, correspond to piecewise-linear segments. If \tri and \trj each correspond to only one line segment, we can simply identify the stationary or boundary points, verify feasibility, and select a solution that yields the optimal cost (derivations for finding the stationary points are provided in the appendix \ref{sec:append_SFT}). However, if either \tri or \trj corresponds to more than one line segment, considering all combinations of segments like in sub-section \ref{sec:C*geometric} will require $O(k_ik_j)$ steps. In this subsection, we present an efficient algorithm that reduces the complexity to  $O(k_i + k_j)$ steps.


\vspace{1mm}
Before we present our algorithm, we mention why solutions to the EFAT and LFDT problems posed in Section \ref{sec:notation} play a crucial role in our approach here. In a EFAT solution, $\mathcal{E}(t_i)$ denotes the earliest feasible time of arrival to the trajectory $\pi_j$ from trajectory point $\pi_i(t_i)$. This implies that the agent cannot reach $\pi_j$ at any time earlier than $\mathcal{E}(t_i)$; also, it is sub-optimal to reach $\pi_j$ at any time later than $\mathcal{E}(t_i)$. Similar observations can also be deduced using a LFDT solution which provides the latest feasible time of departure from a trajectory to a trajectory-point. In the appendix, we show the calculations to solve the EFAT and LFDT problems. Solutions for these problems can be used to prune unnecessary parts of \tri and \trj that will never lead to an optimal solution. Specifically, we can ensure that $\mathcal{E}(\underline{t}_i) = \underline{t}_j$ (or $\underline{t}_i = \mathcal{L}(\underline{t}_j$)) and $\mathcal{E}(\overline{t}_i) = \overline{t}_j$ (or $\overline{t}_i = \mathcal{L}(\overline{t}_j)$). Henceforth, we will assume that the trajectory-intervals already satisfy these conditions. We will also assume that the trajectory-intervals don't intersect each other. That is there is no time $t$ in \timei and \timej such that $\pi_i(t)=\pi_j(t)$; otherwise, this is a trivial case and the optimal SFT cost is 0.

\begin{figure}
    \centering
    \includegraphics[width=1\linewidth]{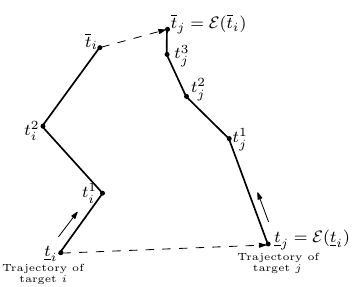}
    \caption{Given target trajectories for {\it AlgoSFT}. In the trajectory of target $i$, there are 3 line-segments and 2 corner points. In the trajectory of target $j$, there are 4 line-segments and 3 corner points. Also, $T_i:=(\underline{t}_i,t_i^1,t_i^2,\overline{t}_i)$ and $T_j:=(\underline{t}_j,t_j^1,t_j^2,t_j^3,\overline{t}_j)$.}
    \label{fig:AlgSFTinput}
\end{figure}


Let $T_i$ and $T_j$ denote lists of selected times (or time instants) corresponding to targets $i$ and $j$. The times in $T_i$ and $T_j$ are automatically sorted in ascending order as new times are added. $T_i$ and $T_j$ are first initialized with all the times corresponding to the corner\footnote{These are break-points where a target trajectory transitions from a line segment to another with a different slope.} trajectory-points in \tri and \trj respectively. In addition, the boundary times $\underline{t}_i, \overline{t}_i$ are added to $T_i$, and similarly, $\underline{t}_j, \overline{t}_j$ are added to $T_j$ (refer to Fig. \ref{fig:AlgSFTinput}). Also, let the $k^{th}$ smallest time in $T_i$ and $T_j$ be referred to as $T_{i}(k)$  and $T_{j}(k)$ respectively. Our algorithm denoted as {\it AlgoSFT} is as follows: 

\begin{enumerate}
    \item For each time $t$ in $T_i$ corresponding to a corner-trajectory point in \tri, find $\mathcal{E}(t)$ and add $\mathcal{E}(t)$ to $T_j$. Similarly, for each time $t$ in $T_j$ corresponding to a corner trajectory-point in \trj, find $\mathcal{L}(t)$ and add $\mathcal{L}(t)$ to $T_i$. At the end of this step, $|T_i|=|T_j|$ (refer to Fig. \ref{fig:AlgoSFTstep1}). 
    \item For $k=1,\cdots, |T_i|-1$, do the following:  Let $\pi_{ik}$ denote the trajectory-sub-interval of \tri corresponding to $[T_{i}(k),T_{i}(k+1)]$. Similarly, let $\pi_{jk}$ denote the trajectory-sub-interval of \trj corresponding to $[T_{j}(k),T_{j}(k+1)]$. Find the SFT cost from $\pi_{ik}$ to $\pi_{jk}$ (using the calculations in the appendix \ref{sec:append_SFT}). Let this cost be denoted as $SFT_k$.
    \item Set $l_{pq}:= \min_{k=1}^{|T_i|-1} SFT_k$.
\end{enumerate}

We will now prove that {\it AlgoSFT} correctly computes the SFT cost from \tri to \trj.

\begin{theorem}\label{thm:monotonic}
Consider any choice of times $t_1,t_2 \in$ \timei such that $t_1<t_2$. Then, $\mathcal{E}(t_1)<\mathcal{E}(t_2)$.
\end{theorem}

\begin{proof}
We prove this theorem by contradiction. Suppose $\mathcal{E}(t_1)\geq \mathcal{E}(t_2)$. Let the position occupied by target $t_2$ at time $\mathcal{E}(t_2)$ be denoted by $p^*=\pi_j(\mathcal{E}(t_2))$. The agent can travel from $\pi_i(t_1)$ to $\pi_i(t_2)$ at speed $v_{max}$ and then travel from $\pi_i(t_2)$ to $\pi_j(\mathcal{E}(t_2))$ following the EFAT path. Since the agent can travel faster than target $i$, it will reach the location $p^*$ sooner than target $t_j$. This will allow the agent to further travel along $\pi_j$ to intercept $t_j$ sooner than $\mathcal{E}(t_2)$. This implies that we have found a new arrival time for the agent that is less than $\mathcal{E}(t_2)\leq \mathcal{E}(t_1)$. As a result $\mathcal{E}(t_1)$ is not the earliest arrival time to visit $t_j$ which is a contradiction. Hence proved.
\end{proof}

\begin{theorem}\label{thm:fixedpoint}
For any time $t \in$ \timei, $\mathcal{L}(\mathcal{E}(t))=t$. Similarly, for any time $t \in$ \timej, $\mathcal{E}(\mathcal{L}(t))=t$.
\end{theorem}
\begin{proof}
We will prove that for any $t \in$ \timei, $\mathcal{L}(\mathcal{E}(t))=t$; the other result can be proved by similar arguments. $\mathcal{L}(\mathcal{E}(t))$ denotes the latest departure time available to leave $\pi_i$ and arrive at $p^*=\pi_j(\mathcal{E}(t))$. Therefore, $\mathcal{L}(\mathcal{E}(t))\geq t$. Now, suppose $\mathcal{L}(\mathcal{E}(t)) > t$. The agent can travel from $\pi_i(t)$ to $\pi_i(\mathcal{L}(\mathcal{E}(t)))$ at speed $v_{max}$ and then travel from $\pi_i(\mathcal{L}(\mathcal{E}(t)))$ to $\pi_j(\mathcal{E}(t))$ following the EFAT path. Similar to the argument in Theorem \ref{thm:monotonic}, if $\mathcal{L}(\mathcal{E}(t)) > t$ is true, the agent will be able to visit target $t_j$ sooner than $\mathcal{E}(t)$ which is not possible. Hence, the only possibility is that $\mathcal{L}(\mathcal{E}(t))= t$.
\end{proof}

\begin{theorem}
    {\it AlgoSFT} correctly finds the optimal SFT cost from \tri to \trj in the order of $k_i+k_j$ steps where $k_i$ and $k_j$ denote the number of line segments in \tri and \trj respectively.
\end{theorem}

\begin{proof}
Consider any time $\bar{t}$ in the list $T_j$ at the end of step 1 of {\it AlgoSFT}. Either $\bar{t}=\mathcal{E}(t)$ for some $t \in T_i$ or there is a time $t\in T_i$ such that $\mathcal{L}(\bar{t})=t$ in which case $\mathcal{E}(t)=\bar{t}$ from Theorem \ref{thm:fixedpoint}. Therefore, for any time $\bar{t}\in T_j$, there is a $t\in T_i$ such that $\mathcal{E}(t)=\bar{t}$. Similarly, for any $t\in T_i$, there is a $\bar{t}\in T_j$ such that $\mathcal{E}(t)=\bar{t}$.

If $T_i(k)$ and $T_j(k)$ denote the $k^{th}$ smallest time in $T_i$ and $T_j$ respectively, Theorem \ref{thm:monotonic} implies that $T_j(k)=\mathcal{E}(T_i(k))$. Suppose the time to depart in an optimal SFT solution from \tri to \trj is $t^*\in$ \timei. Then, for some $k^*\in 1,\cdots, |T_i|-1$, $t^*\in [T_i(k^*),T_i(k^*+1)]$. Applying Theorem \ref{thm:monotonic}, it then must follow that $\mathcal{E}(t^*)\in [T_j(k^*),T_j(k^*+1)]$. Therefore, using the notations in step 2 of {\it AlgoSFT}, the optimal SFT cost from \tri to \trj must be equal $\min_{k=1}^{|T_i|-1} SFT_k$. 

The computation of the optimal SFT cost essentially involves solving $|T_i| - 1$ optimization problems, each requiring the calculation of a fixed number of stationary or boundary points and checking their feasibility. Additionally, we can verify that $|T_i| = |T_j| = k_i + k_j$. Therefore, the number of steps required to implement {\it AlgoSFT} is in the order of $k_i + k_j$. Hence proved. \end{proof}

\begin{figure}
    \centering
    \includegraphics[width=.8\linewidth]{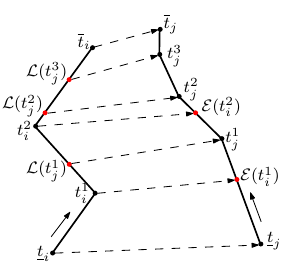}
    \caption{Target trajectories from Fig. \ref{fig:AlgSFTinput} showing the updated lists of times in $T_i$ and $T_j$. At the end of the step 1 of {\it AlgoSFT}, $|T_i|=|T_j|=7$.}
    \label{fig:AlgoSFTstep1}
\end{figure}

\color{black}


\section{Numerical Results}\label{results}
\subsection{Test Settings and Instance Generation} \label{instances}
All the tests were run on a laptop with an Intel Core i7-7700HQ 2.80GHz CPU, and 16GB RAM. For {\blue all the \abbrCstar variants, relaxing MT-TSP and constructing the graph, including computing the edge costs, were all implemented using Python 3.11.6. An exact branch-and-cut solver, written in C++ and utilizing CPLEX 22.1, was used to solve an integer program for the GTSP, to optimality. For the special case where targets move along straight lines, the SOCP-based formulation by Stieber and F{\"u}genschuh in \cite{stieber2022}, with the objective modified to minimize travel time, was used as our baseline. This formulation was implemented in CPLEX 22.1 IDE, which uses OPL. More details on this formulation can be found in the appendix (section~\ref{appendix})}. The CPLEX parameter {\it EpGap}\footnote{Relative tolerance on the gap between the best solution objective and the best bound found by the solver.} was set to be 1e-04 and the CPLEX parameter {\it TiLim}\footnote{Time limit before which the solver terminates.} was set to 7200s {\blue for our algorithms as well as the baseline}.

\vspace{1mm}
{\blue A total of 90 instances were generated, with 30 \emph{\abbrLinInst} instances where targets move along lines, 30 \emph{\abbrPWLinInst} instances where targets move along piecewise-linear paths, and 30 \emph{\abbrDubinsInst} instances where targets move along Dubins curves made of straight lines and circular arcs. 
The \abbrDubinsInst instances are considered separately in section~\ref{section:generic}. For a given instance type, we generated three sets of 10 instances: one set with 5 targets, another set with 10 targets, and a third set with 15 targets. The instances were defined by the number of targets $n$, a square area of fixed size $100 \, units$ containing the start locations of the targets, a fixed time horizon $T = 100 \, secs$ over which the target trajectories are defined, the depot location fixed at the bottom-left corner of the square area with coordinates $(10,10)$, a fixed maximum agent speed of $v_{max}=4 \, units/sec$, and a set of randomly generated trajectories for the $n$ targets, where each target moves at a constant speed within $[0.5,1] \, units/sec$. Each target was also assigned up to 2 time-windows, whose total duration adds up to $20 \, secs$. Note that for the \abbrLinInst instances, we assigned only 1 time-window for each target. We also ensured that the paths traversed by the targets were all confined within the square area. This was so that the baseline SOCP, which relies on these assumptions, could be used.}


\vspace{1mm}
{\blue The time-windows for any given instance were defined as follows. First, the time-window for each target was set to be the entire time horizon, and a feasible solution was found using the algorithm in section~\ref{Sec:feasible solution}. Second, each target was assigned a primary time-window of duration $15 \, secs$, which contains the time that target was visited by the agent in the feasible solution. Third, a secondary time-window of 5 $secs$ that does not intersect with the primary time-window, was randomly assigned to each of the targets as well. For the \abbrLinInst instances, only the primary time-window was assigned to each target, but with an increased duration of $20 \, secs$.}

\vspace{1mm}
{\blue Before proceeding further, note that when finding feasible solutions or when running the \abbrCstar variants, we partition the time-windows for each target into equal intervals of size $\Delta=0.625 \, secs$. Since for any target, the duration of each time-window is an integer multiple of $5 \, secs$, and the total duration from all the time-windows sums to $20 \, secs$, we get a total of $32$ intervals per target. This is always true unless otherwise specified. Also, when using \abbrCstarSampling, the sampling parameter \abbrSamplingPar is always set to 10, and the gap tolerance \abbrSamplingGapTol is set to 0.05.}

\subsection{Finding Feasible Solutions}\label{Sec:feasible solution}
{\blue We evaluate the quality of bounds from the \abbrCstar variants, based on how far, they deviate on average, from feasible solution costs. This section briefly discusses how feasible solutions for the MT-TSP can be obtained by first transforming MT-TSP into a corresponding GTSP and then finding feasible solutions for the GTSP.}


\vspace{1mm}
{\blue The time-windows for each target are first sampled into equally spaced time-instants. The trajectory-points corresponding to these time-instants are then found. A directed graph $\mathcal{G}$ is then constructed, with the vertex set defined as the depot and the set of all trajectory-points found. All the vertices corresponding to a given target are clustered together. If the agent can travel feasibly from a vertex belonging to a cluster to another vertex belonging to a different cluster, a directed edge is added, with the cost being the difference between the time-instants corresponding to the destination vertex and the start vertex. If the destination vertex is the depot, then the edge cost is the time taken by the agent to reach the depot from the start vertex by moving at its maximum speed. If travel between two vertices is not feasible, the cost of the edge between the vertices is set to a large value to denote infeasibility.}


\vspace{1mm}
A feasible solution for the MT-TSP can now be obtained by finding a directed edge cycle which starts at the depot vertex, visits exactly one vertex from each cluster, and returns to the depot vertex. Note that this is simply, the problem of finding a feasible solution for the GTSP defined on graph $\mathcal{G}$. One way to solve this problem is to first transform the GTSP into an Asymmetric TSP (ATSP) using the transformation in \cite{noon1993efficient}, and then find feasible solutions for the ATSP using an LKH solver \cite{helsgaun2000effective} or some other TSP heuristics. {\blue Note that one can also directly use heuristics for GTSP such as GLKH \cite{helsgaun2015solving}, and GLNS \cite{smith2017glns}.} Since the target trajectories are continuous functions of time, the best arrival times for each target can be calculated, given the order in which they were visited in the GTSP solution, resulting in a feasible tour with improved travel time. Note that if the number of discrete time-instants are not sufficient, we may not obtain a feasible solution for the MT-TSP using this approach, even if one exists.

\subsection{Evaluating the Bounds} \label{linear}
{\blue In this section, we compare, for all the \abbrLinInst and \abbrPWLinInst instances, the feasible solution costs and the lower-bounding costs from the \abbrCstar variants. For \abbrLinInst instances, the optimum from the baseline SOCP is also added to the comparison.}



\vspace{1mm}
{\blue The results are presented in Fig.~\ref{costs}. Here, (a), (b), (c) includes all the \abbrLinInst instances, and (d), (e), (f) includes all the \abbrPWLinInst instances. The instances are sorted from left to right in the order of increasing feasible solution costs. We observe \abbrCstarLinear provides the tightest bounds, followed by \abbrCstarGeometric, then \abbrCstarSampling, and finally \abbrCstardash. The bounds from \abbrCstarLinear are the strongest since it uses the optimum for SFT. However, these bounds are closely matched by those from \abbrCstarGeometric and \abbrCstarSampling. Note that although \abbrCstarGeometric relaxes the timing requirements for SFT, it returns slightly stronger bounds than \abbrCstarSampling. The bounds from \abbrCstarSampling however improves, and converges to those from \abbrCstarLinear as the sampling parameter \abbrSamplingPar approaches $\infty$ and the gap tolerance \abbrSamplingGapTol approaches 0. This however, comes at the expense of increased computational burden. Finally, \abbrCstardash uses trivial lower-bounds for SFT, making its bounds the weakest.
For \abbrLinInst instances, feasible solution costs are tightly bounded by the SOCP costs, with the lower-bounds not exceeding the SOCP costs. This shows that the approach to find feasible solutions is effective, and that the \abbrCstar variants indeed provide lower-bounds for the MT-TSP.

\vspace{1mm}
For instance-1 in (c), the SOCP cost slightly exceeds the feasible cost. This is because the problem becomes significantly more computationally expensive at 15 targets, and as a result, the CPLEX solver failed to converge to the optimum within the time limit for that instance. Hence, the best feasible cost found by the solver before exceeding the time limit was used for this instance. Similarly, in (c), the bounds from all \abbrCstar variants except \abbrCstardash are weaker for instance-1, and the bounds from all \abbrCstar variants are weaker for instance-5. These too, were due to the increased computational complexity when considering 15 targets. Here, the CPLEX solver terminated due to insufficient memory, leaving the gap between the dual bound and the best objective value, not fully converged to be within the specified tolerance. In such cases, the best lower-bound found by the solver before termination was considered, as it still provides an underestimate. These \emph{outlier} instances are illustrated using square markers, indicating that CPLEX solver terminated due to memory constraints.}

\begin{figure*}
    \centering
    \includegraphics[width=1\textwidth]{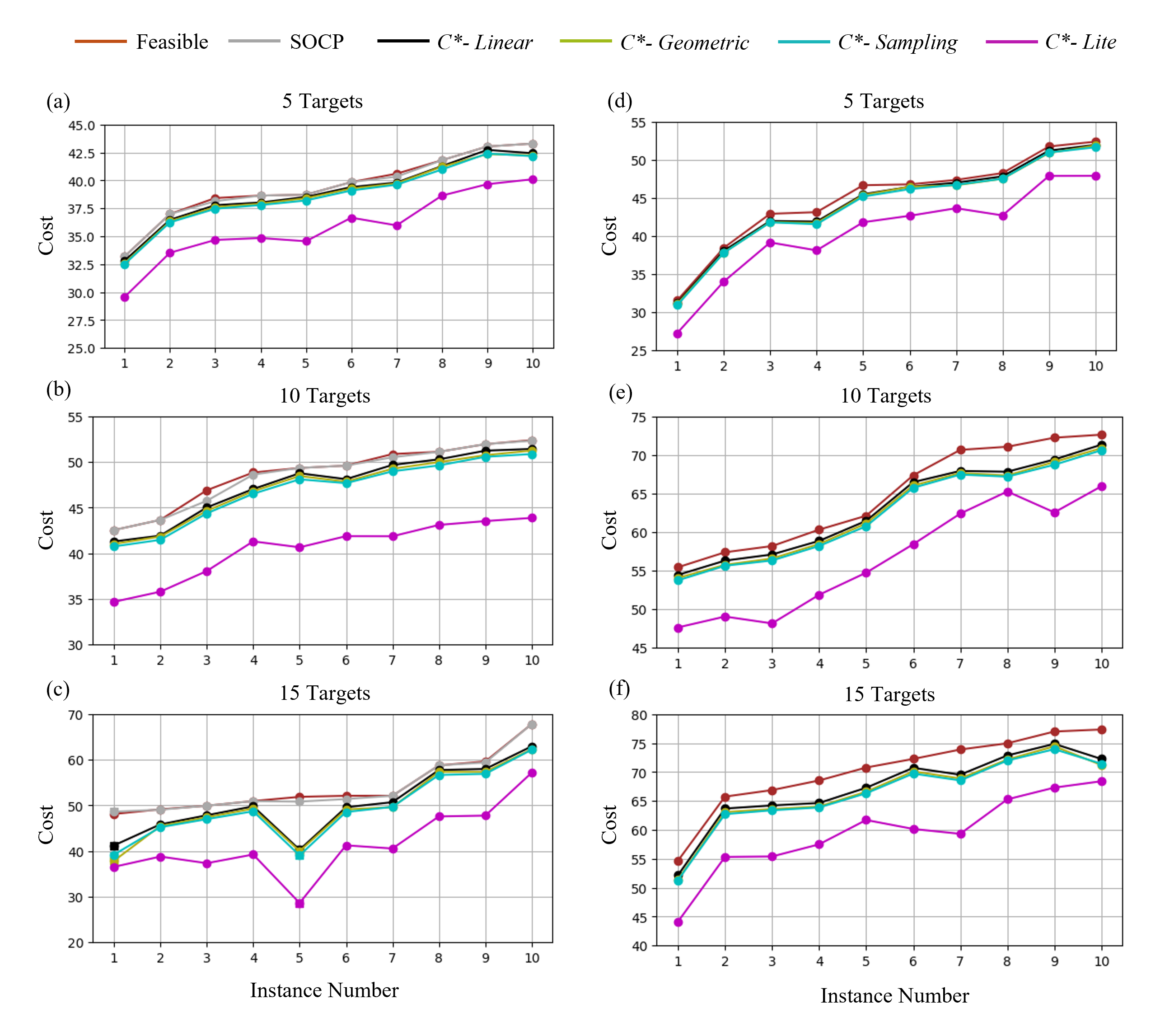} 
    \caption{{\blue Lower-bounds from the \abbrCstar variants as compared to the feasible costs, for the \abbrLinInst instances (left), and the \abbrPWLinInst instances (right). The SOCP costs are also included for the \abbrLinInst instances. In (c), square markers indicate the outlier cases for instances 1, and 5.}}
    \label{costs}
\end{figure*}






\subsection{Varying the Number of Targets}
{\blue In this section, we see how varying the number of targets affects the tightness of bounds, as well as runtimes, for the \abbrCstar variants. The SOCP is also evaluated for \abbrLinInst instances. For a given instance, the \abbrperDev is defined as $\frac{C_f-C_{lb}}{C_f}\times 100$ where $C_f$ denotes the feasible solution cost and $C_{lb}$ denotes the lower-bound. The lower-bound here represents the cost returned from a \abbrCstar variant, or the SOCP. The runtime (\abbrRT) for an instance is separated into the graph generation runtime (\abbrGRT), which is the runtime for constructing graphs in the \abbrCstar variants, and the total runtime (\abbrtotalRT), which is the time taken by the \abbrCstar variants for graph generation, and then solving GTSP on the generated graph. The runtime for SOCP is also referred to as \abbrtotalRT. In Fig.~\ref{devvstar}, (a) considers all the \abbrLinInst instances and (b) considers all the \abbrPWLinInst instances. The average of the \abbrperDev from all the instances (except outliers) corresponding to 5, 10, and 15 targets are presented in the figure. Fig.~\ref{rtvstar} follows the same layout as Fig.~\ref{devvstar}, and presents runtimes instead of \abbrperDev}.

\begin{figure*}
    \centering
    \includegraphics[width=\textwidth]{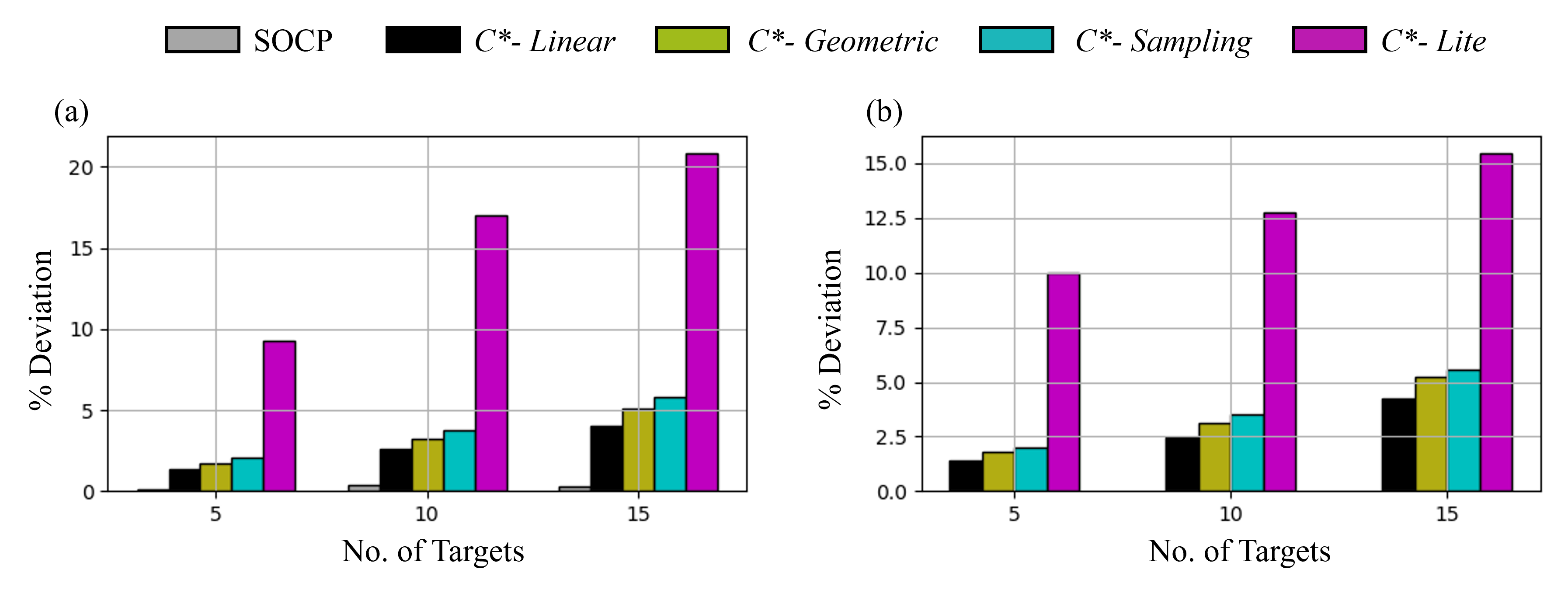}  
    \caption{{\blue Plots comparing the average \abbrperDev for the \abbrCstar variants (and SOCP for \abbrLinInst instances), as the number of targets are varied. Note how the \abbrperDev for \abbrCstarLinear is $\approx 4 \%$ for 15 targets, for both the \abbrLinInst (a) and \abbrPWLinInst (b) instances. Also, note how the average \abbrperDev for SOCP is less than $1 \%$.}}
    \label{devvstar}
\end{figure*}

\begin{figure*}
    \centering
    \includegraphics[width=\textwidth]{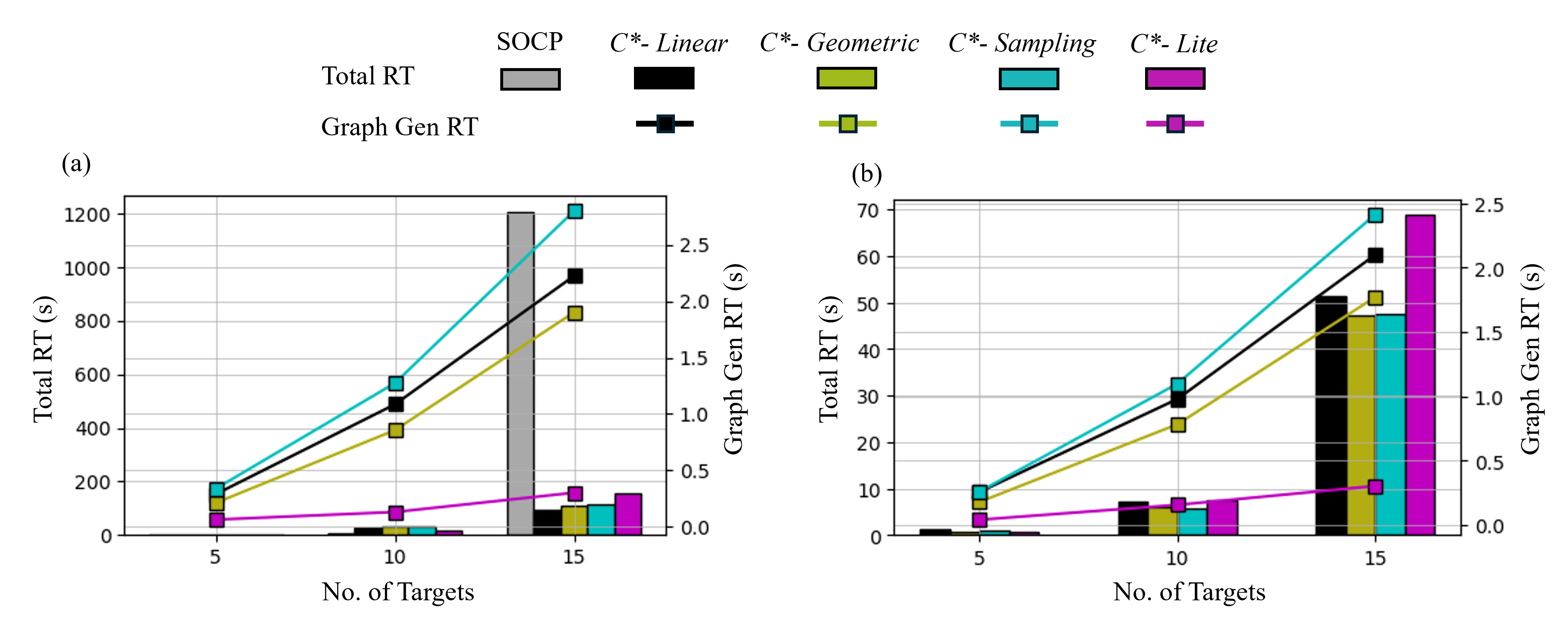} 
    \caption{ { \blue Plots comparing the average runtimes for the \abbrCstar variants, as the number of targets are varied, for both the \abbrLinInst (a), and \abbrPWLinInst (b) instances. Note how for the \abbrLinInst instances, SOCP is faster than the \abbrCstar variants for up to 10 targets, but an order of magnitude slower for 15 targets. }}
    \label{rtvstar}
\end{figure*}

\vspace{1mm}
{\blue From both Fig.~\ref{devvstar} (a) and (b), we observe that \abbrperDev for the \abbrCstar variants increases as the number of targets are increased. This is to be expected since each visited target in the relaxed MT-TSP incurs an additional discontinuity. As we saw previously, \abbrCstarLinear gives the best bounds, with an average \abbrperDev $\approx 4 \%$ for 15 targets. This is followed by \abbrCstarGeometric, \abbrCstarSampling, and \abbrCstardash, respectively. We also observe that the growth in \abbrperDev with the number of targets is greater in \abbrCstardash as compared to the other \abbrCstar variants.
Note how the \abbrperDev for \abbrCstarLinear, \abbrCstarGeometric, and \abbrCstarSampling for 15 targets, are still smaller than that of \abbrCstardash for 5 targets.
Finally, we observe the \abbrperDev for SOCP to be very small, indicating that the feasible solution costs obtained were on average, very close to the optimal costs. Note that the SOCP costs does not depend on the number of targets as it aims to find the optimum.}



\vspace{1mm}
{\blue From Fig.~\ref{rtvstar} (a) and (b) we see that \abbrGRT is the smallest for \abbrCstardash, followed by \abbrCstarGeometric, then \abbrCstarLinear, and finally, \abbrCstarSampling. 
This is because the lower-bounding algorithm for SFT in \abbrCstardash is trivial, making it the fastest. This is followed by the one in \abbrCstarGeometric, which solves an easier time-independent problem. \abbrCstarLinear finds the optimum for the SFT making it slightly slower. However, it uses \abbrsftAlg, which is specifically tailored for piecewise-linear target trajectories, and involves an efficient search. Finally, for \abbrCstarSampling, the computational burden for finding lower-bounds for SFT increases with larger sampling parameter \abbrSamplingPar, and smaller gap tolerance, \abbrSamplingGapTol. By setting \abbrSamplingPar to $10$ and \abbrSamplingGapTol to 0.05, \abbrCstarSampling incurred more computational burden than the other \abbrCstar variants. 
We also observe that the rate at which \abbrGRT grows for the different \abbrCstar variants can be explained similarly. 
Note however, that \abbrGRT in general contributes very little to the \abbrtotalRT as compared to the time taken to solve the GTSP. We see a big jump in \abbrtotalRT with increasing targets, especially between 10 targets and 15 targets. This can be attributed to the increasing complexity of solving GTSP on larger graphs. 
We observe that on average, the \abbrtotalRT is the least for \abbrCstarGeometric and \abbrCstarSampling, followed by \abbrCstarLinear, and then \abbrCstardash. It is likely that the relatively slower runtimes for \abbrCstardash is due to its weaker relaxation for MT-TSP, making the underlying GTSP more difficult to solve.
Finally, we observe from Fig.~\ref{rtvstar} (a), how the \abbrtotalRT for SOCP is significantly smaller than the \abbrCstar variants for 5, and 10 targets, but becomes an order of magnitude larger for 15 targets.}

\subsection{Varying the Discretization}
{\blue In this section, we show how varying the discretization levels for \abbrCstar variants affect their \abbrperDev and runtimes. Recall that for any target, the duration of each time-window is an integer multiple of $5 \, secs$, and the total duration from all the time-windows sums to $20 \, secs$. Considering this, we define the discretization levels as follows. At \texttt{lvl-1}, the time-windows for each target are partitioned into equal intervals of duration of $5 \, secs$. For every new level thereafter, the interval durations are halved, and the number of intervals are doubled. All of this is illustrated in Table.~\ref{tablediscrt}.

\begin{table}
    \centering
    \begin{tabular}{|c|c|c|c|c|}
        \hline
         Discretization Level & lvl-1 & lvl-2 & lvl-3 & lvl-4 \\
        \hline 
         Intervals per Target & 4 & 8 & 16 & 32 \\
        \hline
         Interval Duration & 5 & 2.5 & 1.25 & 0.625 \\
        \hline 
    \end{tabular}
    \vspace{2mm}
    \caption{Information about the discretization levels.}
    \label{tablediscrt}
\end{table}
}




\vspace{1mm}
{\blue Fig.~\ref{devrtvsdiscrt} (a) and (b) considers all the \abbrLinInst and \abbrPWLinInst instances, except for the outlier instances previously discussed. From all the instances considered, (a) illustrates the average \abbrperDev, and (b) illustrates the average runtimes, for each discretization level. From (a), we observe how higher discretization produces tighter bounds for all the \abbrCstar variants as one would expect. Note that the \abbrperDev improvement for \abbrCstardash is more significant here, as compared to the rest of the variants. Also, note that the \abbrperDev for \abbrCstardash at \texttt{lvl-4} is still higher than it is for \abbrCstarLinear at \texttt{lvl-1}. Finally, note that the rate at which \abbrperDev improves, decreases at higher discretizations, for all the approaches. From (b), we observe how the runtimes increase with higher discretization. 
Note that the increase in \abbrGRT is more significant for \abbrCstarLinear, \abbrCstarGeometric, and \abbrCstarSampling, than it is for \abbrCstardash here.
Finally, we observe how the growth in \abbrtotalRT is significantly higher between \texttt{lvl-3}, and \texttt{lvl-4}, than it is between the previous levels. Like before, this too can be attributed to the increasing complexity of solving GTSP on larger graphs. Sample solutions for a {\it complex} instance are shown in Fig. \ref{fig:linear}.}

\begin{figure*}[t]
    \centering
    \includegraphics[width=\textwidth]{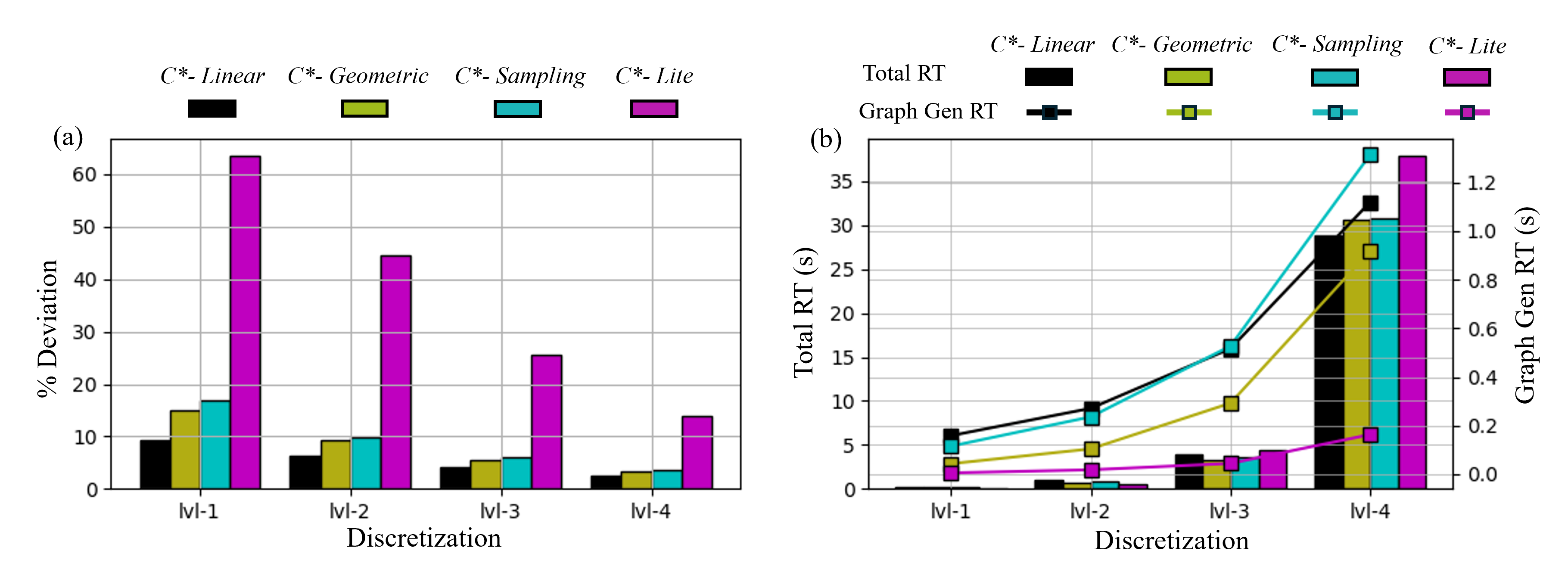}  
    \caption{{ \blue Plots illustrating (a) the average \abbrperDev for the \abbrCstar variants, and (b) the average runtimes for the \abbrCstar variants, for different levels of discretization. Note how in (a), \abbrCstarLinear at \texttt{lvl-1} has lower \abbrperdev than \abbrCstardash at \texttt{lvl-4}. Also in (b), note how at \texttt{lvl-4}, the runtimes increase significantly for all approaches.}}
    \label{devrtvsdiscrt}
\end{figure*}



\subsection{Obtaining Feasible Solutions from \abbrCstar Variants}

{\blue In this section, we attempt to construct feasible solutions for the MT-TSP, from lower-bounds obtained from the \abbrCstar variants. We also evaluate how good the costs are for these new solutions. Clearly, if the discretization parameter $\Delta$ goes to $0$, the lower-bounds converge to the optimum. However, this is computationally infeasible and therefore, we will fix the discretization level at \texttt{lvl-4}.}


{\blue
\vspace{1mm}
To construct feasible solutions from lower-bounds, we first fix the order in which the targets are visited in the lower-bounding solution, and then find a minimum cost tour for the agent over that fixed order such that it a) visits each target within one of its time-windows and b) completes the tour without exceeding its maximum speed $v_{max}$. For the cases where such a tour cannot be constructed, we say a feasible solution cannot be constructed from a given lower-bound. Note that this is the same procedure we used in section~\ref{Sec:feasible solution} to obtain feasible tours with improved travel times from the ones initially found by solving the GTSP. 
}


{\blue
\vspace{1mm}
For Table.~\ref{tablesuccess}, we consider all the \abbrLinInst and \abbrPWLinInst instances, except for the outlier instances (same instances as in Fig.~\ref{devrtvsdiscrt}). Here, {\it Success Rate} illustrates the percentage of instances from which a feasible solution can be constructed from the lower-bound. For such instances, we compare the new feasible solutions with the original ones to determine if the order in which the targets are visited remains the same (or matches). Note that if the orders match, then the arrival times for the targets must also be the same, since we used the same procedure to reoptimize the original feasible solutions, as well as construct new feasible solutions from lower-bounds, as discussed earlier. The percentage of instances where the feasible solutions match is given by $\%$ Match in the table. Finally, from all the remaining instances where the solutions do not match, we average the costs of both the feasible solutions originally obtained as well as the newly constructed ones, denoted by $C_f^{old}$ and ${C_f^{new}}$ respectively, and find $\%$ Dev-Mismatch defined as $\frac{C_f^{new}-C_f^{old}}{C_f^{old}}\times100$.
}


{\blue
\vspace{1mm}
We observe that the success rate is more than 95$\%$ for all the \abbrCstar variants, with \abbrCstardash giving the best rate of 98$\%$. 
If a feasible solution cannot be constructed from one lower-bound, it might still be possible to construct one from another. Hence, for at least 98$\%$ of the instances, we were able to construct feasible solutions from various \abbrCstar lower-bounds. We also observe that the $\%$ match is similar for \abbrCstarLinear, \abbrCstarGeometric, and \abbrCstarSampling, and it is significantly higher than for \abbrCstardash. Finally, we observe that the $\%$ dev-mismatch remains close to 0$\%$ for all the \abbrCstar variants. Specifically, ${C_f^{new}}$ is marginally worse than ${C_f^{old}}$ for \abbrCstarLinear and \abbrCstardash, while ${C_f^{new}}$ is marginally better than ${C_f^{old}}$ for \abbrCstarGeometric and \abbrCstarSampling.
}


\begin{table}[h]
    \centering
    \begin{tabular}{|c|c|c|c|c|}
    \hline
         & Success Rate ($\%$) & $\%$ Match & $\%$ Dev-Mismatch  \\
    \hline
         \abbrCstarLinear & 96.55 & 64.29 & 0.21 \\
    \hline
         \abbrCstarGeometric & 96.55 & 60.34 & -0.026 \\
    \hline
         \abbrCstarSampling & 96.55 & 62.07 & -0.072 \\
    \hline
         \abbrCstardash & 98.28 & 47.37 & 0.27 \\
    \hline
    \end{tabular}
    \vspace{2mm}
    \caption{Table illustrating the success rates for constructing feasible solutions from various \abbrCstar lower-bounds, and how these new feasible solutions compare with the originally found feasible solutions.}
    \label{tablesuccess}
\end{table}

{\blue
\subsection{Evaluating \abbrCstar Variants for Generic Instances} \label{section:generic}


\begin{figure*}
    \centering
    \includegraphics[width=\textwidth]{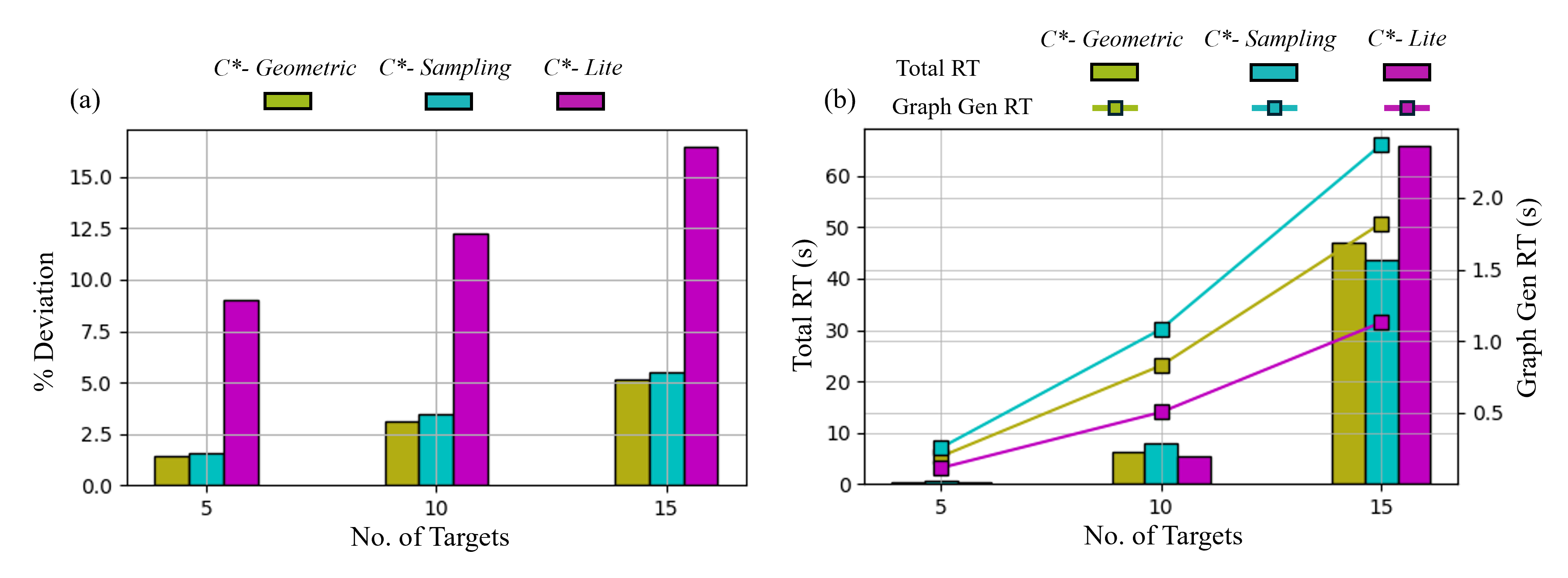}
    \caption{{\blue Plots illustrating (a) the average \abbrperDev, and (b) the average runtimes, for the \abbrCstar variants, as the number of targets are varied. Only the \abbrDubinsInst instances are considered for both the plots.}}
    \label{pltDubinsResults}
\end{figure*}

In this section, we consider only the 30 generic instances, where targets move along Dubins curves. We evaluate \abbrCstarGeometric, \abbrCstarSampling, and \abbrCstardash for these instances, and omit \abbrCstarLinear as it specifically caters to targets moving along piecewise-linear paths. Fig.~\ref{pltDubinsResults} illustrates in (a), the average \abbrperDev, and in (b), the average runtimes, for 5, 10, and 15 targets. Like for \abbrLinInst and \abbrPWLinInst instances in Fig.~\ref{devvstar}, the \abbrperDev here grows with more targets, with \abbrCstarGeometric always giving the best bounds, followed by \abbrCstarSampling, and then \abbrCstardash. Here too, the \abbrperDev for both \abbrCstarGeometric and \abbrCstarSampling at 15 targets are still smaller than it is for \abbrCstardash at 5 targets. Note that the \abbrperDev for \abbrCstarGeometric at 15 targets is $\approx 5 \%$. Averaging this with the \abbrperDev of $\approx 4 \%$ from \abbrPWLinInst instances, our approaches give on average, a \abbrperDev of $\approx 4.5 \%$ for general cases of MT-TSP. The runtime results are also similar to the ones presented for \abbrLinInst and \abbrPWLinInst instances in Fig.~\ref{rtvstar}. Both \abbrGRT and \abbrtotalRT increases with more number of targets. The \abbrGRT as well as its growth, are more significant for \abbrCstarGeometric and \abbrCstarSampling  than it is for \abbrCstardash. Finally, the \abbrtotalRT grows significantly for all the \abbrCstar variants considered, as the number of targets are increased from 10 to 15. Sample solutions for the general case are shown in Fig. \ref{fig:dubins}.
}


\begin{figure*}[htb!]
\centering
  \begin{subfigure}[t]{.32\linewidth}
    \centering\includegraphics[width=1\linewidth]{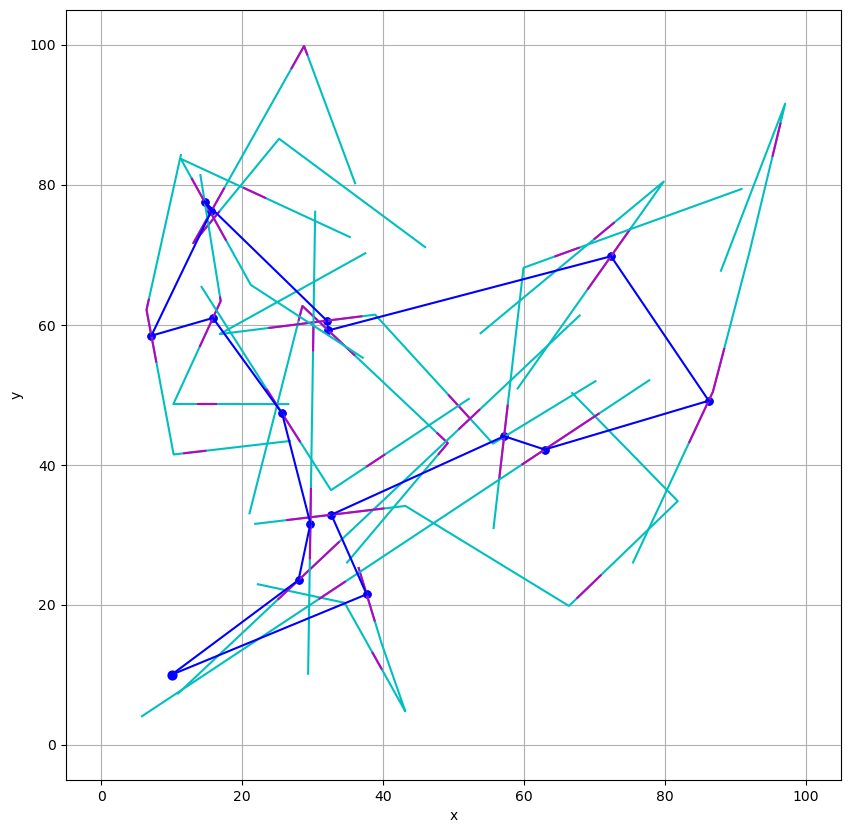}
    \caption{Feasible solution}
    \label{fig:linearFeasible}
  \end{subfigure} 
  \begin{subfigure}[t]{.32\linewidth}
    \centering\includegraphics[width=1\linewidth]{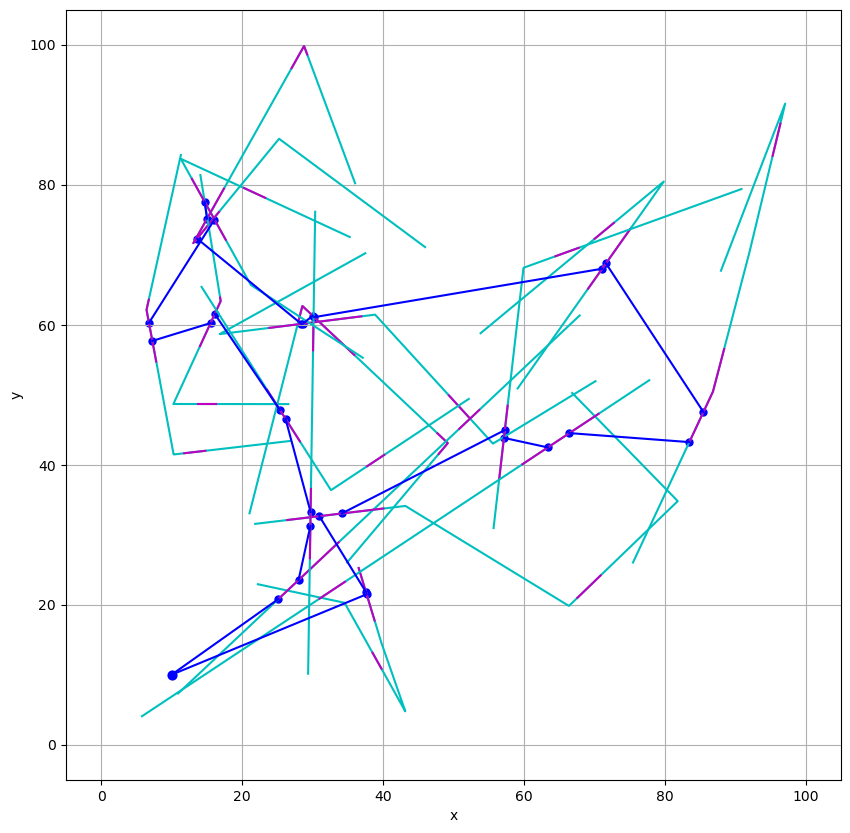}
    \caption{\abbrCstarLinear}
     \label{fig:linearLinear}
  \end{subfigure} 
  \begin{subfigure}[t]{.32\linewidth}
    \centering\includegraphics[width=1\linewidth]{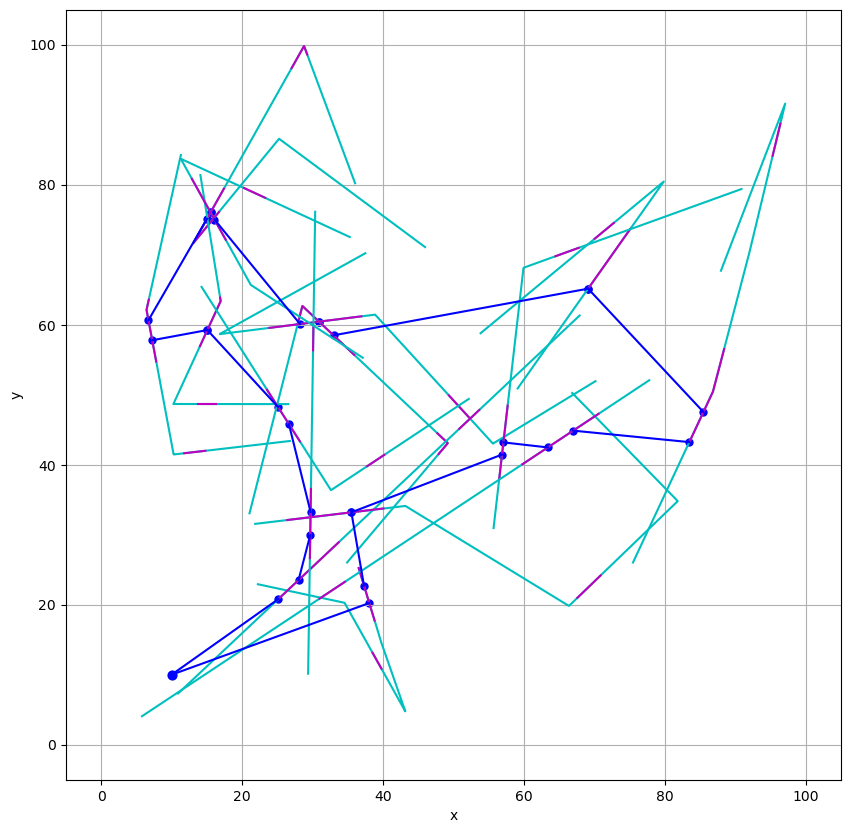}
    \caption{\abbrCstarGeometric}
     \label{fig:linearGeometric}
  \end{subfigure}
  \caption{A feasible solution (using the algorithm in section \ref{Sec:feasible solution}) and the two best bounding solutions for a complex instance with 15 targets moving along piecewise-linear paths.}
  \label{fig:linear}
\end{figure*}

\begin{figure*}[htb!]
\centering
  \begin{subfigure}[t]{.32\linewidth}
    \centering\includegraphics[width=1\linewidth]{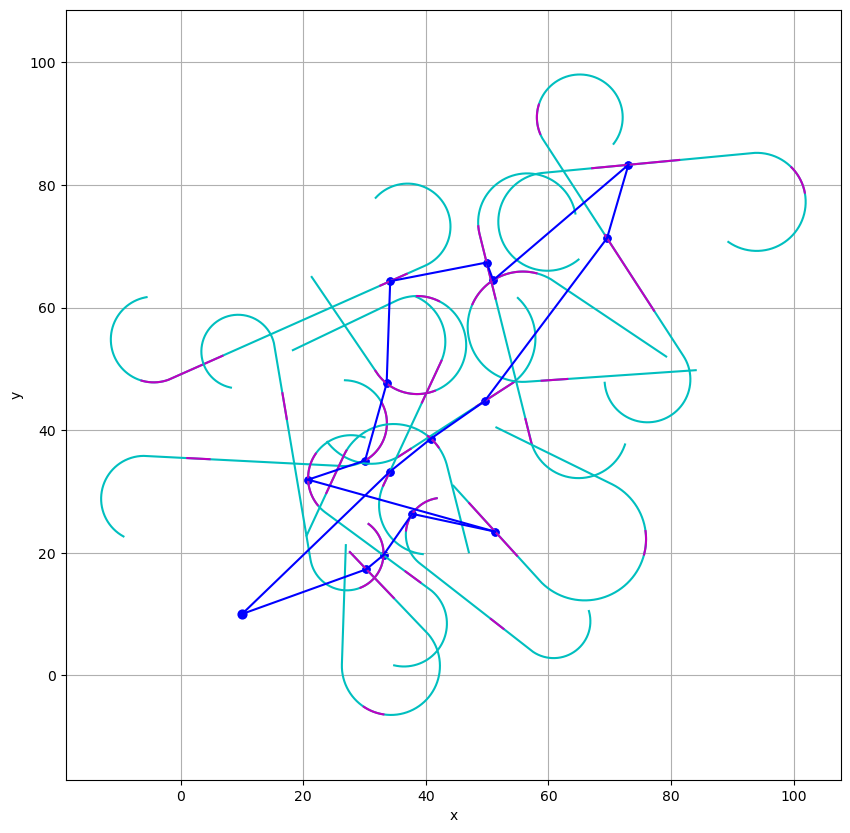}
    \caption{Feasible solution}
    \label{fig:dubinsFeasible}
  \end{subfigure} 
  \begin{subfigure}[t]{.32\linewidth}
    \centering\includegraphics[width=1\linewidth]{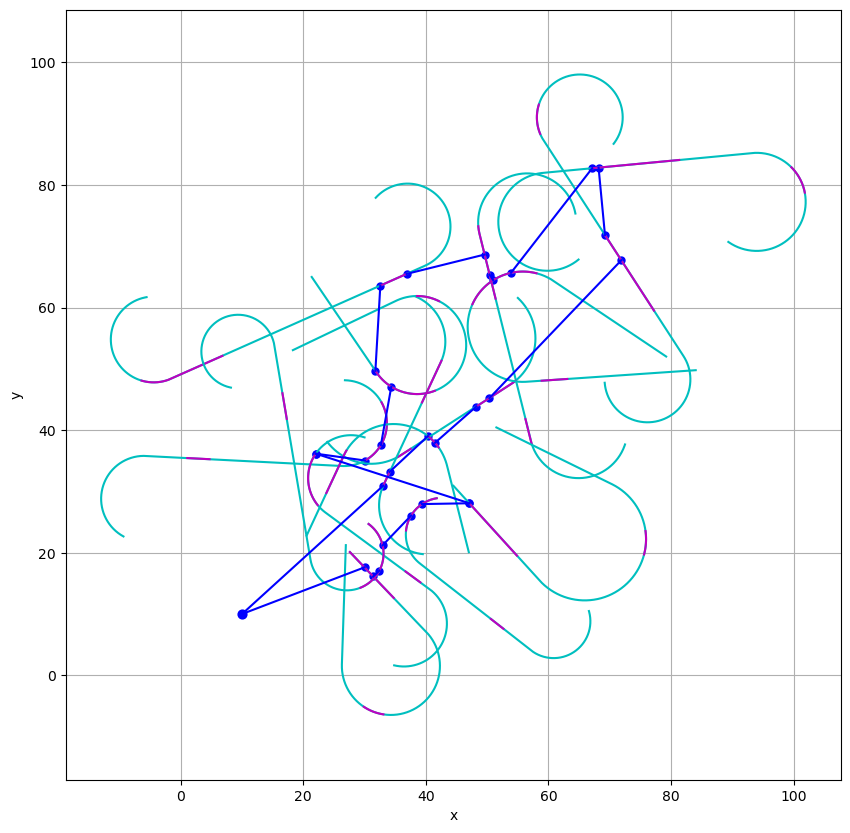}
    \caption{\abbrCstarGeometric}
     \label{fig:dubinsGeometric}
  \end{subfigure} 
  \begin{subfigure}[t]{.32\linewidth}
    \centering\includegraphics[width=1\linewidth]{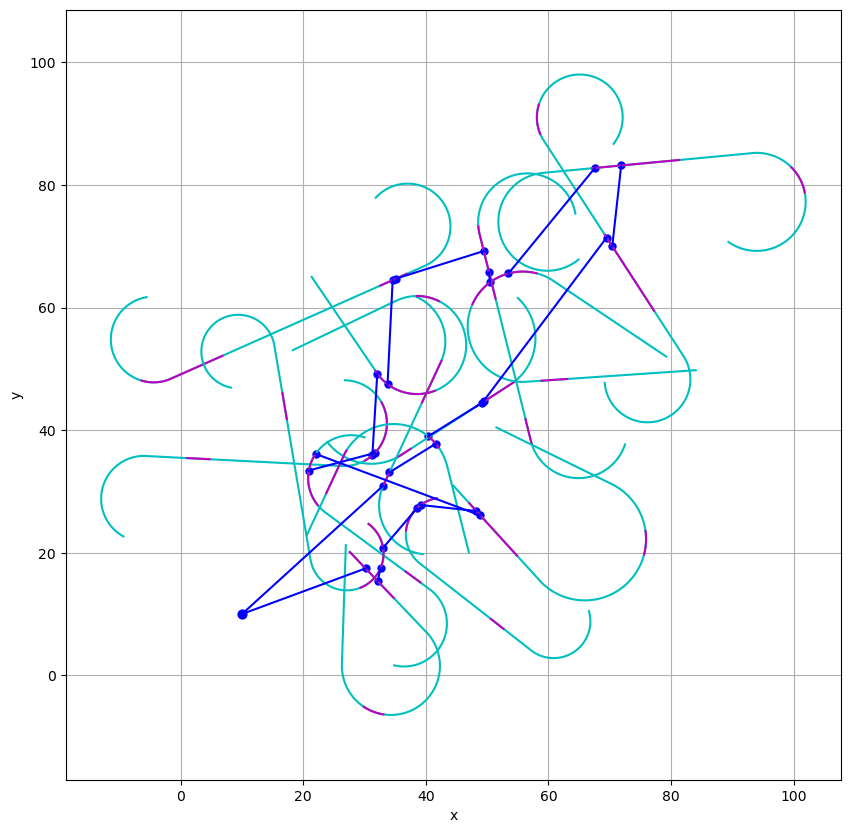}
    \caption{\abbrCstarSampling}
     \label{fig:dubinsSampling}
  \end{subfigure}
  \caption{A feasible solution (using the algorithm in section \ref{Sec:feasible solution}) and the two best bounding solutions for a general instance with 15 targets moving along Dubins curves.}
  \label{fig:dubins}
\end{figure*}

\section{Conclusion and Future Work}\label{conclude}
{\blue We presented \abbrCstarVars, an approach for finding lower-bounds for the MT-TSP with time-window constraints. Our method can handle generic target trajectories, including piecewise-linear segments and Dubins curves. We introduced several variants of \abbrCstarVars, providing a trade-off between the quality of guarantees and algorithm's running speed. Additionally, we proved that our approaches yield valid lower-bounds for the MT-TSP and presented extensive numerical results to demonstrate the performance of all the variants of \abbrCstarVars.} Finally, we showed that feasible solutions can often be constructed from the lower-bounding solutions obtained using \abbrCstarVars variants.

One of the challenges in this paper was the computational burden associated with increasing the number of targets. This difficulty arises from solving the GTSP, where the computational complexity heavily depends on the number of nodes in the generated graph. Finding a way to overcome this challenge would enable us to compute tight bounds for a larger number of targets. {\blue In fact, one of our primary future goals is to develop efficient branch-and-cut implementations that leverage the specific features of the MT-TSP problem. For instance, unlike standard GTSP formulations that include subtour elimination constraints, we may be able to omit some of these constraints because the timing constraints in our problem inherently eliminate certain subtours (e.g., those that travel back in time). This represents a promising new research direction that warrants further investigation. Another direction of research can focus on developing approximation algorithms for simpler variants of the MT-TSP. }

\bibliographystyle{plain}

\section{Appendix} \label{appendix}
\subsection{Finding Earliest Feasible Arrival Time} \label{sec:append_EFA}
\begin{figure}[h]
    \centering
    \includegraphics[width=\linewidth]{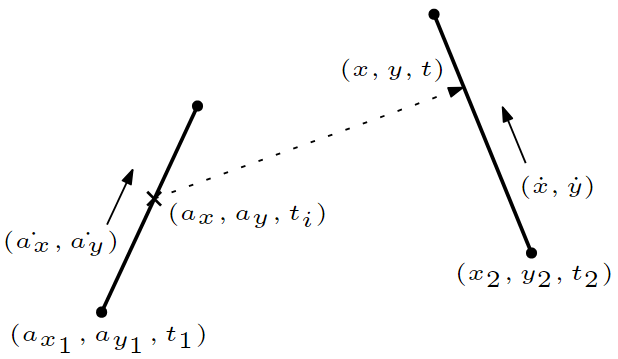}
    \caption{Trajectories $\pi_i$ and $\pi_j$ for targets $i$ and $j$.}
    \label{eg3}
\end{figure}

Let the trajectory-point $\pi_{s}(t)$ for some target $s$ at time $t$ be denoted by the tuple $(x, y)$, where $x$ and $y$ are the coordinates of the position occupied by $s$ at time $t$. Also, let $\dot{x}$ and $\dot{y}$ denote the time derivative of $x$ and $y$ respectively. Let $i$ and $j$ be two targets moving along trajectories $\pi_{i}$ and $\pi_{j}$ as shown in Fig~\ref{eg3}. Note that $\pi_{i}(t_i) \equiv (a_{x}, a_{y})$ and $\pi_{j}(t) \equiv (x,y)$. Let $\pi_{i}(t_1) \equiv (a_{x_1}, a_{y_1})$ and $\pi_{j}(t_2) \equiv (x_2, y_2)$. The following equations then describe the motion of $i$ and $j$ from times $t_1$ and $t_2$ onward, respectively.
\begin{align}
   &a_{x} = a_{x_1} + \dot{a_x}(t_i - t_1) \label{m-1} \\ 
   &a_{y} = a_{y_1} + \dot{a_y}(t_i - t_1) \label{m-2} \\
   &x = x_2 + \dot{x}(t - t_2) \label{m-3} \\
   &y = y_2 + \dot{y}(t - t_2) \label{m-4}
\end{align}


\vspace{1mm}
The distance between $(a_x, a_y)$ and $(x,y)$ is then given by $\sqrt{(x-a_x)^2+(y-a_y)^2}$. Also, we want the agent to travel at speed $v_{max}$ to obtain $e(t_i)$. Hence, we want $t$ that satisfies
\begin{align}
    &\sqrt{(x-a_x)^2+(y-a_y)^2} = v_{max}(t-t_i) \label{sqrtreq-1}
\end{align}

Squaring both sides, we obtain
\begin{align}
    &(x-a_x)^2+(y-a_y)^2 = v_{max}^2(t-t_i)^2 \label{req-1}
\end{align}

Substituting \eqref{m-1}, \eqref{m-2}, \eqref{m-3}, \eqref{m-4} into \eqref{req-1}, and rearranging the terms, we get
\begin{align}
    &(\dot{x}t+C_1)^2+(\dot{y}t+C_2)^2 = v_{max}^2(t-t_i)^2
\end{align}

where
\begin{align*}
    &C_1 = -\dot{a_x}t_i+C_1' \\
    &C_2 = -\dot{a_y}t_i+C_2' \\
    &C_1' = x_2-\dot{x}t_2-a_{x_1}+\dot{a_x}t_1 \\
    &C_2' = y_2-\dot{y}t_2-a_{y_1}+\dot{a_y}t_1
\end{align*}

After extensive algebra, we finally get the following.
\begin{align}
    &At^2+B(t_i)t+C(t_i) = 0 \label{eqEFA}
\end{align}

where
\begin{align*}
    &A = \dot{x}^2+\dot{y}^2-v_{max}^2 \\
    &B(t_i) = 2B't_i+2C' \\
    &C(t_i) = A't_i^2-D't_i+E' \\
    &B' = -\dot{a_x}\dot{x}-\dot{a_y}\dot{y}+v_{max}^2 \\
    &C' = C_1'\dot{x} + C_2'\dot{y} \\
    &A' = \dot{a_x}^2+\dot{a_y}^2-v_{max}^2 \\
    &D' = 2\dot{a_x}C_1'+2\dot{a_y}C_2' \\
    &E' = C_1'^2 + C_2'^2
\end{align*}

One of the two roots that satisfies \eqref{eqEFA} is then the \abbrEFAT $\mathcal{E}(t_i)$. These roots\footnote{We obtain two roots since \eqref{sqrtreq-1} is squared to get \eqref{req-1}. However, only one of the roots yield the \abbrEFAT. Similar reasoning can be used to explain the same occurrence when finding the \abbrLFDT.} can be obtained using the quadratic formula as shown below.
\begin{align} \label{eqEFAquad}
    t = \frac{-B(t_i) \pm \sqrt{B(t_i)^2 - 4AC(t_i)}}{2A}
\end{align}

\subsection{Finding Latest Feasible Departure Time} \label{sec:append_LFD}

From Theorem \ref{thm:fixedpoint}, we know that $t = \mathcal{E}(t_i) \iff \mathcal{L}(t) = t_i$. Hence, given a value of $t$, we seek to find $t_i$ such that $t$ satisfies \eqref{eqEFA}. To solve this, note that \eqref{eqEFA} can be expanded as follows.
\begin{align} \label{eqexpEFA}
    At^2+(2B't_i+2C')t+(A't_i^2-D't_i+E')=0
\end{align}

By rearranging the terms in \eqref{eqexpEFA} we get the below equation.
\begin{align} \label{eqexpLFD}
    A't_i^2+(2B't-D')t_i+(At^2+2C't+E')=0
\end{align}

\eqref{eqexpLFD} can then be represented simply as 
\begin{align} \label{eqLFD}
    A't_i^2+H(t)t_i+I(t)=0
\end{align}

Hence, one of the two roots that satisfies \ref{eqLFD} is the \abbrLFDT $\mathcal{L}(t)$. Like previously shown, these roots can be obtained using the quadratic formula below.
\begin{align}
    t_i = \frac{-H(t) \pm \sqrt{H(t)^2 - 4A'I(t)}}{2A'}
\end{align}

\subsection{Finding the Stationary Points for the SFT problem}\label{sec:append_SFT}
Consider the same setup as in \ref{sec:append_EFA}. In the SFT problem, our aim is to find a feasible travel from $t_i$ such that $\mathcal{E}(t_i)-t_i$ is minimized. While in the EFAT problem, $t_i$ is given, here $t_i$ is also a variable and the arrival time $t$ will be a function of $t_i$. This function has already been derived in Equation \eqref{eqEFAquad}. By differentiating $t$ in \eqref{eqEFAquad} with respect to $t_i$, we can find an expression for $\frac{dt}{dt_i}$ as follows.

\begin{align*}
    &\frac{dt}{dt_i}=\frac{1}{2A}\left[-\frac{d}{dt_i}B(t_i) \pm \frac{2B(t_i)\frac{d}{dt_i}B(t_i)-4A\frac{d}{dt_i}C(t_i)}{2\sqrt{B(t_i)^2-4AC(t_i)}}\right] = \\
    &\frac{1}{2A}\left[-2B' \pm \frac{2(2B't_i+2C')(2B')-4A(2A't_i-D')}{2\sqrt{(2B't_i+2C')^2-4A(A't_i^2-D't_i+E')}} \right]
\end{align*}

\vspace{1mm}
After further simplification, we get
\begin{align} \label{eqdtdt_i}
    \frac{dt}{dt_i}=\frac{1}{A}\left[-B' \pm \frac{2B'(B't_i+C')-A(2A't_i-D')}{2\sqrt{(B't_i+C')^2-A(A't_i^2-D't_i+E')}} \right]
\end{align}

\vspace{1mm}
To find the stationary points of $t-t_i$, we set $\frac{d}{dt_i}(t-t_i)=0$ which then gives us
\begin{align} \label{eqmain}
    &\frac{dt}{dt_i}-1=0\implies \frac{dt}{dt_i}=1
\end{align}

\vspace{1mm}
Substituting \eqref{eqdtdt_i} into \eqref{eqmain}, we get the following.
\begin{align*}
    &\frac{1}{A}\left[-B' \pm \frac{2B'(B't_i+C')-A(2A't_i-D')}{2\sqrt{(B't_i+C')^2-A(A't_i^2-D't_i+E')}} \right]=1
\end{align*}

\vspace{1mm}
which can be rearranged as
\begin{align}
    (A+B')=\pm \frac{2B'(B't_i+C')-A(2A't_i-D')}{2\sqrt{(B't_i+C')^2-A(A't_i^2-D't_i+E')}} \label{eqdengeq0}
\end{align}

\vspace{1mm}
Note that $(B't_i+C')^2-A(A't_i^2-D't_i+E') \geq 0$ for all $t_i$. If $(B't_i+C')^2-A(A't_i^2-D't_i+E')=0$ for some $t_i$, then it means both $s_1$ and $s_2$ occupies the same position at time $t_i$. In this case, the function $t-t_i$ becomes $0$ and takes a sharp turn ($\frac{d}{dt_i}(t-t_i)$ becomes undefined) at $t_i$. However, if $(B't_i+C')^2-A(A't_i^2-D't_i+E')>0$, we get the following by squaring both sides of \eqref{eqdengeq0} and multiplying the denominator on both sides.
\begin{align*}
    &4(A+B')^2((B't_i+C')^2-A(A't_i^2-D't_i+E')) \\
    =&(2B'(B't_i+C')-A(2A't_i-D'))^2
\end{align*}

\vspace{1mm}
After extensive algebra, we finally get the following.
\begin{align} \label{eq_statquad}
    Pt_i^2+Qt_i+R=0
\end{align}

\vspace{1mm}
where
\begin{align*}
\begin{split}
    P&=4(A+B')^2(B'^2-AA')-4B'^4 \\
    &-(4A^2A'^2-8AA'B'^2)
\end{split}\\
\begin{split}
    Q&=4(A+B')^2(2B'C'+AD')-8B'^3C' \\
    &-(4AB'^2D'-8AA'B'C'-4A^2A'D')
\end{split}\\
    R&=4(A+B')^2(C'^2-AE')-4B'^2C'^2 \\
    &-(4AB'C'D'+A^2D'^2)\\
\end{align*}

\vspace{1mm}
The values of $t_i$ that satisfies the two equations in \eqref{eqdengeq0} can be obtained by solving for the two roots that satisfies \eqref{eq_statquad}. These roots can once again, be obtained using the quadratic formula given below.
\begin{align}
    t_i=\frac{ -Q \pm \sqrt{Q^2-4PR}}{2P}
\end{align}

\subsection{Second Order Cone Program (SOCP) Formulation}
In this section, we explain the SOCP formulation for the special case of the MT-TSP where targets follow linear trajectories. Our formulation is very similar to the one presented in \cite{stieber2022}, with a few changes made to accommodate the new objective which is to minimize the time taken by the agent to complete the tour, as opposed to minimizing the length of the path traversed by the agent.

\vspace{1mm}
To ensure that the trajectory of the agent starts and ends at the depot, we do the following: Given the depot $d$ and a set of $n-1$ targets $\{1, \cdots, n-1\}$, we define a new stationary target $n$ which acts as a copy of $d$. This is achieved by fixing $n$ at the same position as $d$. We then define constraints so that the agent's trajectory starts from $d$ and ends at $n$. To find the optimal solution to the MT-TSP, we then seek to minimize the time at which $n$ is visited by the agent.

\vspace{1mm}
Let $S \equiv \{1, \cdots, n\}$ and $S_d \equiv \{d\} \cup S$. We use the same family of decision variables as in \cite{stieber2022} where $x_{i,j} \in \{0,1\}$ indicates the decision of sending the agent from target $i$ to target $j$ ($x_{i,j}=1$ if yes. No otherwise), and $t_{i} \in \mathbb{R}$ describes the arrival time of the agent at target $i$ (or depot).

\vspace{1mm}
Our objective is to minimize the time at which the agent arrives at target $n$ as shown below.
\begin{align}
    \min t_{n}
\end{align}

\vspace{1mm}
Each target $j$ must be visited once by the agent:
\begin{align}
    \sum_{i \in S_d: i \neq j}x_{i,j}=1, \; \forall \: j \in S 
\end{align}

\vspace{1mm}
The agent can start only once from the depot:
\begin{align}
    \sum_{j \in S}x_{d,j}\leq 1
\end{align}

\vspace{1mm}
Flow conservation is ensured by:
\begin{align}
    \sum_{i \in S_d: i \neq j}x_{i,j} \geq \sum_{i \in S: i \neq j}x_{j,i}, \; \forall \: j \in S
\end{align}

\vspace{1mm}
The agent must visit each target within its assigned time-window. Note that, for the depot $d$ and the target $n$, we assign the time-window to be the entire time-horizon $T$:
\begin{align}
    t_l^j \leq t_{j} \leq t_u^j, \; \forall \: j \in S_d
\end{align}

\vspace{1mm}
As shown in \cite{stieber2022}, real auxiliary variables $c_{i,j}^x$ and $c_{i,j}^y$ for the $x-$ and $y-$ components of the Euclidean distance are introduced as follows:
\begin{align}
\begin{split}
    &c_{d,j}^x-\left(\left(x_l^j+t_{j}\frac{\Delta{x_j}}{\Delta{t_j}}-t_l^j\frac{\Delta{x_j}}{\Delta{t_j}}\right)- d_x \right)=0, \\
    & \forall \: j \in S
\end{split}    
\end{align}
\begin{align}
\begin{split}
    &c_{d,j}^y-\left(\left(y_l^j+t_{j}\frac{\Delta{y_j}}{\Delta{t_j}}-t_l^j\frac{\Delta{y_j}}{\Delta{t_j}}\right)- d_y \right)=0, \\
    & \forall \: j \in S
\end{split}    
\end{align}
\begin{align}
\begin{split}
    c_{i,j}^x &- \left(\left(x_l^j+t_{j}\frac{\Delta{x_j}}{\Delta{t_j}}-t_l^j\frac{\Delta{x_j}}{\Delta{t_j}}\right)- \left(x_l^i+t_{i}\frac{\Delta{x_i}}{\Delta{t_i}}-t_l^i\frac{\Delta{x_i}}{\Delta{t_i}}\right)\right) \\
    &=0, \; \forall \: i \in S, j \in S: i \neq j
\end{split}
\end{align}
\begin{align}
\begin{split}
    c_{i,j}^y &- \left(\left(y_l^j+t_{j}\frac{\Delta{y_j}}{\Delta{t_j}}-t_l^j\frac{\Delta{y_j}}{\Delta{t_j}}\right)- \left(y_l^i+t_{i}\frac{\Delta{y_i}}{\Delta{t_i}}-t_l^i\frac{\Delta{y_i}}{\Delta{t_i}}\right)\right) \\
    &=0, \; \forall \: i \in S, j \in S: i \neq j
\end{split}
\end{align}

\vspace{1mm}
Here, for some target $i$, $(x_l^i,y_l^i)$ represents the coordinates of $i$ at time $t_l^i$ and $(x_u^i,y_u^i)$ represents the coordinates of $i$ at time $t_u^i$. Also, $\Delta{x_i}=x_u^i-x_l^i$, $\Delta{y_i}=y_u^i-y_l^i$, and $\Delta{t_i}=t_u^i-t_l^i$. Finally, $(d_x,d_y)$ denotes the coordinates of the depot $d$.

\vspace{1mm}
The following conditions requires that if the agent travels between any two targets or the depot and a target, this travel must be feasible:
\begin{align}
\begin{split}
    &a_{i,j} \leq v_{max}(t_{j}-t_{i}+T(1-x_{i,j})), \\
    &\forall \: i \in S_d, j \in S : i \neq j \label{eq_MTZ}
\end{split} \\
    &a_{i,j} \geq 0, \; \forall \: i \in S_d, j \in S : i \neq j
\end{align}

\vspace{1mm}
The below conditions are needed to formulate the cone constraints:
\begin{align}
\begin{split}
    &\overline{a}_{i,j}=a_{i,j}+R(1-x_{i,j}), \; \forall \: i \in S_d, j \in S : i \neq j
\end{split}
\end{align}

\vspace{1mm}
Where given the square area with fixed side length $L$ that contains all the moving targets and the depot, $R=\sqrt{2}L$ is the length of the square's diagonal.

\vspace{1mm}
Finally, the cone constraints are given as:
\begin{align}
    &(c_{i,j}^x)^2+(c_{i,j}^y)^2 \leq (\overline{a}_{i,j})^2, \; \forall \: i \in S_d, j \in S : i \neq j 
\end{align}

\vspace{1mm}
The agent must visit $n$ only after visiting all the other targets:
\begin{align}
    t_{n} \geq t_{j}, \; \forall \: j \in S_d
\end{align}

\begin{remark}
    Although \eqref{eq_MTZ} prevents subtours in most cases, they can still arise very rarely when two or more target trajectories intersect at a time common to their time-windows. The constraints defined by \eqref{eq_sub_elim_two_tar} prevents subtours for the two target case. However for more than two targets, we will need additional subtour elimination constraints:
    \begin{align}
        x_{i,j}+x_{j,i} \leq 1, \; \forall \: i \in S, j \in S: i \neq j \label{eq_sub_elim_two_tar}
    \end{align}
\end{remark}

\end{document}